\documentclass[11pt]{article}

\usepackage[top=50mm, bottom=50mm, left=50mm, right=50mm]{geometry}
\usepackage{lineno}
\usepackage{amssymb}
\usepackage{amsmath}
\usepackage{amsthm}
\usepackage{epsfig}
\usepackage{graphicx}
\usepackage{graphics}
\usepackage{float}
\usepackage{subfigure}
\usepackage{multirow}
\usepackage{color}
\usepackage{lineno}
\usepackage{fullpage}
\usepackage[normalem]{ulem} 
\usepackage{makeidx}
\usepackage{xspace}
\usepackage{wrapfig}
\makeindex

\newtheorem{lemma}{Lemma}

\definecolor{darkred}{rgb}{1, 0.1, 0.3}
\definecolor{darkblue}{rgb}{0.1, 0.1, 1}
\definecolor{darkgreen}{rgb}{0,0.6,0.5}

\newcommand {\mm}[1] {\ifmmode{#1}\else{\mbox{\(#1\)}}\fi}

\usepackage{graphicx}
\usepackage[utf8]{inputenc} % allow utf-8 input
\usepackage[T1]{fontenc}    % use 8-bit T1 fonts
\usepackage{hyperref}       % hyperlinks
\usepackage{url}            % simple URL typesetting
\usepackage{booktabs}       % professional-quality tables
\usepackage{amsfonts}       % blackboard math symbols
\usepackage{nicefrac}       % compact symbols for 1/2, etc.
\usepackage{microtype}      % microtypography
\usepackage{amssymb}
\usepackage{amsmath}
\usepackage{wrapfig}
\newtheorem{proposition}{Proposition}
\newtheorem{assumption}{Assumption}
\DeclareMathOperator{\sgn}{sgn}
\usepackage{algorithm,algorithmic}

\usepackage{hyperref}
\hypersetup{colorlinks=true,allcolors=[rgb]{0,0.25,0.65}}
\allowdisplaybreaks[1]
\begin{document}

\title{\textbf{Learning to Initialize Gradient Descent Using Gradient Descent}}
 
\author{
Kartik Ahuja$^{*}$ \and Amit Dhurandhar$^{*}$ \and Kush R. Varshney\footnote{IBM Research, Thomas J. Watson Research Center, Yortown Heights, New York} 
}
\date{}

\maketitle

\begin{abstract}
Non-convex optimization problems are challenging to solve; the success and computational expense of a gradient descent algorithm or variant depend heavily on the initialization strategy. Often, either random initialization is used or initialization rules are carefully designed by exploiting the nature of the problem class. As a simple alternative to hand-crafted initialization rules, we propose an approach for learning  ``good'' initialization rules from previous solutions. We provide theoretical guarantees that establish conditions that are sufficient in all cases and also necessary in some under which our approach performs better than random initialization. We apply our methodology to various non-convex problems such as generating adversarial examples, generating post hoc explanations for black-box machine learning models, and allocating communication spectrum, and show consistent gains over  other initialization techniques.
\end{abstract}
\section{Introduction}

In many machine learning and engineering tasks, we are often required to  solve optimization problems repeatedly. For instance, a system that generates explanations for black-box machine learning models \cite{ribeiro2016should} \cite{dhurandhar2018explanations} needs to solve a new optimization problem for every new prediction made; a recommender system solves a different optimization problem \cite{koren2009matrix,mairal2010online} every time a new subject arrives. Many of these optimization problems are non-convex, and initialization plays a crucial rule in finding a good local minimum. As is typically the case, the optimization problems for each instance are solved independently. However, if one were to learn ``good'' initializations for many of the instances based on previous solutions, it could lead to significant savings in time, money, and energy, where function and gradient evaluation has an associated cost and carbon footprint \cite{strubell2019energy}.  In this work, we develop methods that \emph{learn to initialize} based on past knowledge.

 Over the recent years, several works, an incomplete representative list --  \cite{andrychowicz2016learning}
 \cite{li2016learning}
 \cite{wichrowska2017learned}
 \cite{finn2018learning}\cite{ khalil2017learning} -- have explored  the ``learning to optimize'' paradigm. Departing from expert-driven design, these works study data-driven design of optimization algorithms for training machine learning models. They are inspired by the long line of work in meta-learning \cite{thrun2012learning, schmidhuber1987evolutionary,schmidhuber1992learning}. Gradient descent based algorithms \cite{nocedal2006numerical} typically consist of two blocks: initialization and step updates (update direction, and learning rates). Works such as \cite{andrychowicz2016learning,li2016learning,wichrowska2017learned,finn2018learning} primarily focus on learning step updates instead of following the standard step update rules.  Many works such as \cite{antoniou2018train,flennerhag2018transferring,finn2017model} focus on learning good models that serve as good initialization for different few-shot learning tasks. In these works \cite{antoniou2018train,flennerhag2018transferring,finn2017model}, the parametrization/identity of the learning task is not known. Hence, these works adopt an approach that does not require a task's identity as input to compute the initialization. However, in many optimization problems that are of interest to us (e.g., generating explanations and adversarial examples), we are given the identity of the optimization problem (e.g., prediction to be explained); we take advantage of this additional information herein. 
 
 Hand-designing initializers tailored to a specific problem class is common.  For instance,  in clustering  \cite{arthur2007k,pena1999empirical}, phase retrieval \cite{candes2015phase},  and deep learning \cite{glorot2010understanding,mishkin2015all, dauphin2019metainit}, different elegant initialization rules have been proposed for the respective problem classes that work better than random initialization. However, it is impractical to develop these rules for every new optimization problem we encounter and a more scalable strategy is highly desirable.  Hence, we aim to build a data-driven initialization approach that is scalable, adaptable, and learns  initialization rules based on the identity of the optimization task.

We propose two methods. The first method uses objective function values whereas the second method uses argument values at the solution to learn the initializers. We establish mild conditions under which the methods are guaranteed to perform well and better than the random initialization. The first method is designed for non-convex problems, especially where the variance in local minimum values is large, and does not offer advantage in convex problems. The second method is designed to work in both convex and non-convex problems. We carry out extensive experiments to show that the proposed methods perform better than many of the existing approaches on several convex and non-convex optimization problems.

\section{Problem formulation}
\label{secn: prob_formln}
 We are given an objective function $f:\Theta \times \mathcal{X} \rightarrow \mathbb{R}$ to be minimized. The input to the function is divided in two categories:
  \begin{itemize}
 \item $\theta \in \Theta \subseteq \mathbb{R}^{m}$:  the set of variables to be optimized, 
\item $x\in \mathcal{X} \subseteq \mathbb{R}^{n}$:  the parameters that specify the identity of the optimization instance.
 \end{itemize}
 The minimization problem for instance $x$ is given as
\begin{equation}
  \min_{\theta \in \Theta} f(\theta, x),
  \label{eqn: stat_opt}
\end{equation}
 For each $x\in \mathcal{X}, $ the function $f(\cdot,x)$ is differentiable in $\theta$ for all $\theta$ in the interior of $\Theta$. We use a gradient descent solver (see Algorithm \ref{alg1:GDA}), which we refer to as $\mathsf{GD}$. $\mathsf{GD}$ takes as input an initialization $\theta_{\mathsf{in}}$ and function $f(\cdot,x)$ corresponding to the instance $x$, and outputs the argument and function value at the solution. Define   $g:\Theta \times \mathcal{X} \rightarrow \mathbb{R}$  to capture the relationship between the initialization $\theta_{\mathsf{in}}$, problem instance $x$, and the function value at the solution $g(\theta_{\mathsf{in}},x)$ found by $\mathsf{GD}$.  Similarly define  $g^{\dagger}$, where $g^{\dagger}(\theta_{\mathsf{in}},x) \in \Theta$ is the argument value at the solution. If  the number of iterations in $\mathsf{GD}$ is large and certain standard conditions  are met (gradient of $f(\cdot,x)$ is Lipschitz continuous and step updates follow Wolfe conditions) \cite{nocedal2006numerical}, then $\mathsf{GD}$ converges to a \emph{stationary point}. In such cases, $g$ and $g^{\dagger}$ are the function value and argument at the stationary point.

We are presented a sequence of instances $\{X_i\}_{i=1}^{\infty}$,  drawn i.i.d.\ from a distribution $\mathbb{P}_X$. The sequence of functions corresponding to these instances are $\{f_i\}_{i=1}^{\infty}$, where $f_i = f(\cdot,X_i)$ is a function of $\theta$ for a fixed $X_i$.  Our goal is to solve equation \eqref{eqn: stat_opt} repeatedly for $\{X_i\}_{i=1}^{\infty}$. Observe that we can use the information  that we acquire from solving the optimization  for past instances to improve the current search. We now provide motivating examples.
 
 \textbf{ML applications.}
 Consider binary classification  with labels in $\{1,-1\}$.
 Given a model $u:\mathbb{R}^{m} \rightarrow \mathbb{R}$, prediction $y=\sgn(u(x))$. Consider the following problem: 
\begin{equation}
  \min_{\theta \in \Theta}\;\;\; \|\theta\|^2  + \beta \|\theta\|_1  \quad 
  \text{s.t.}\;\;\;  y(u(x+\theta))\leq 0.
  \label{eqn: stat_opt_example1}
\end{equation}

If we set $\beta=0$ and $\Theta =\mathbb{R}^{m}$, the problem is finding adversarial example for an instance $x$ input to model $u$ \cite{goodfellow2014explaining}. $\theta$ is the smallest perturbation (attack) that moves $x$ to the other side of the boundary $u(x)=0$.  If $u$ is a deep neural network, then the above problem is non-convex and it is non-trivial to find the global optimum tractably. There is a lot of interest in understanding \cite{cheng2018query,cheng2019sign} whether an adversary can generate attacks with  few queries to the model and its gradients. By exploiting the past adversarial instances, we show  that the adversary can generate good adversarial examples with few queries.

If $\beta>0$ and the set $\Theta$ is constrained to a special non-convex set referred to as ``pertinent negatives'' (PNs)  \cite{dhurandhar2018explanations}, the problem in equation \eqref{eqn: stat_opt_example1} is equivalent to finding instance-wise model explanations in the form of PNs.  For a system that generates explanations on a repeated instance-to-instance basis (such as Fiddler AI, IBM Watson Openscale, H2O), it is important to generate these explanations fast and with a few model queries to minimize the user's costs.

\textbf{Engineering application.} In a communication system \cite{chiang2007power}, users often transmit signals in the same spectrum and interfere with each other. Therefore, a base station must determine a spectrum sharing protocol. Consider a system with $N$ senders and $N$ receivers. Each sender $i$ transmits at power level $\theta^i \in [0,1]$. Define a channel matrix $x$, where the channel strength for sender $i$ to the receiver $j$ is $x[i,j] \in [0,1]$. The rate at which sender $i$'s data is transmitted to receiver $i$ is given as $r_{i}(\mathbf{\theta}, x) = \log(1 + \frac{x[i,i]\theta^i}{1+\sum_{j\not=i}x[j,i] \theta^{j}})$, where $\mathbf{\theta} = [\theta^1,...,\theta^N]$ is the power vector.  
The goal is to solve for $\mathbf{\theta} \in [0,1]^{N}$ maximizing $\sum_{i=1}^{N} r_{i}(\mathbf{\theta}, x)$. This problem is referred to as ``sum-rate optimization'.  In reality, the channel matrix $x$ is not fixed and keeps changing \cite{tse2005fundamentals}. As a result, the  problem needs to be solved repeatedly thus making learning initialization important.

In all three examples described above, currently used optimization solvers repeatedly solve problem instances independently. Next, we describe how we can capture data from previous solves to improve future searches. 

\textbf{Scope of this work.} Before proceeding, let us delineate problems which are not in the scope of current work, but could be fruitful avenues for future research. Our formulation is designed for scenarios in which individual task parametrization $x$ values are known such as  generating explanations, adversarial examples, etc. Consider we are training neural networks parametrized by  $\theta\in \Theta$ for few-shot learning, i.e., finding networks that adapt quicky to a set  prediction tasks $x\in \mathcal{X}$, where $x$ characterizes the joint distribution of the features and the labels. In such cases, the objective defined in equation \eqref{eqn: stat_opt}, $f(\theta,x)$, corresponds to the expected risk of model $\theta$ for task $x$. To apply our method to these tasks, $x$ needs to be estimated, which is beyond the scope of current work. Model-agnostic meta-learning (MAML) \cite{finn2017model} does not use $x$ and thus is better suited for such few-shot learning tasks. We also do not focus on building initializers \cite{glorot2010understanding}, \cite{dauphin2019metainit} that specifically overcome the difficulties of training deep learning models.

\begin{algorithm}[H]
  \caption{Gradient Descent Algorithm $(\mathsf{GD})$}
  \label{alg:example}
\begin{algorithmic}
  \STATE {\bfseries Input:} function $w$,   initial value $\theta_{\mathsf{in}}$, $\epsilon$, $\mathsf{iter_{max}}$, step size rules $\{t_{k}\}_{k=0}^{\infty}$,  $\Pi_{\Theta}$ projection on the set $\Theta$
\STATE Initialize: $\mathsf{iter} = 0$
  \WHILE{$\mathsf{iter} \leq \mathsf{iter_{max}}$ and $\|\nabla_{\theta}w(\theta)\|_{\theta=\theta_{\mathsf{iter}}}\geq \epsilon$}
  \STATE $\theta_{\mathsf{iter}+1} = \Pi_{\Theta} \Big[ \theta_{\mathsf{iter}} -t_{\mathsf{iter}}\nabla_{\theta}w(\theta) \Big]$
\STATE $\mathsf{iter} = \mathsf{iter} + 1$
  \ENDWHILE
  \STATE $\theta^{\dagger}= \theta_{\mathsf{iter}}$, $w^{\dagger} = w(\theta_{\mathsf{iter}})$
  \STATE {\bfseries Output:} $\theta^{\dagger}$, $w^{\dagger}$
  \label{alg1:GDA}
\end{algorithmic}
\end{algorithm}

\section{Learn to initialize  gradient descent}

\subsection{Independent random vs.\ conditional random initialization}
Define a random variable $\hat{\theta}$ with a distribution $\mathbb{P}_{\hat{\theta}}$ with support $\Theta$. Independent random initialization method works as follows. For a problem instance $X$, $\mathsf{GD}$ is initialized using an independent draw of $\hat{\theta}$ and it finds a solution with objective value $\hat{Y} = g(X, \hat{\theta})$ ($g$ was defined in Section \ref{secn: prob_formln}). 
In this work, we develop a general form of initialization where the draw is conditioned on the problem instance $X \sim \mathbb{P}_{X}$.
Define a random variable $\tilde{\theta}$ with a distribution $\mathbb{P}_{\tilde{\theta}|X}$. Conditional random initialization method works as follows. For a problem instance $X$,  $\mathsf{GD}$ is initialized using a conditionally independent draw of $\tilde{\theta}|X$ and it finds a solution with objective value $\tilde{Y} = g(X, \tilde{\theta})$.  The performance of standard random initialization is $\mathbb{E}_{X,\hat{\theta}}[\hat{Y}]$, while that of conditional random initialization  is $\mathbb{E}_{X,\tilde{\theta}}[\tilde{Y}]$.

\textbf{What are the minimum conditions on $\mathbb{P}_{\tilde{\theta}}$ to ensure that conditional random initialization $\tilde{\theta}$ is better than independent random initialization $\hat{\theta}$? }
The answer leads us to the principle underlying our main method. We carry out the analysis for a simple family of initializers $\tilde{\theta}$ that we define next.
Consider two independent random initializations $\hat{\theta}_{0} \sim \mathbb{P}_{\hat{\theta}}$ and $\hat{\theta}_{1} \sim \mathbb{P}_{\hat{\theta}}$ with corresponding solution values  $\hat{Y}_0 = g(X, \hat{\theta}_0)$ and $\hat{Y}_1 = g(X, \hat{\theta}_1)$ respectively. Define a random variable $Z \in \{0,1\}$, where $Z|\hat{\theta}_0, \hat{\theta}_1,X$ is a Bernoulli random variable where $p(\hat{\theta}_0, \hat{\theta}_1,X)$ is the probability $Z=0$. 
We define a family of initializers $\tilde{\theta}$ that select from one of the two initializers $\hat{\theta}_{0}, \hat{\theta}_{1}$  as $\tilde{\theta} = (1-Z)\hat{\theta}_0 + Z\hat{\theta}_1$.  

\begin{assumption}
(Expected probablistic ordering)
$  \mathbb{E}_{X, \hat{\theta}_{0}, \hat{\theta}_{1}}\Big[(p(\hat{\theta}_0, \hat{\theta}_1,X)-\frac{1}{2}) (\hat{Y}_0 - \hat{Y}_1) \Big] <0 $.
\label{assm1}
\end{assumption}

If Assumption \ref{assm1} holds, then  for ``some realizations'' of $\hat{\theta}_0$, $\hat{\theta}_1$ and $X$, we know either  $p(\hat{\theta}_0, \hat{\theta}_1,X) > \frac{1}{2}, \hat{Y}_0 < \hat{Y}_1 $ or $p(\hat{\theta}_0, \hat{\theta}_1,X)< \frac{1}{2}, \hat{Y}_0 > \hat{Y}_1$. For such realizations,  $\tilde{\theta}$ correctly orders the two initializations with probability more than $\frac{1}{2}$, which is just better than a random guess.

\begin{proposition} (Conditional  vs.\ independent random)
 $\mathrm{Assumption\; \ref{assm1}}\iff$  $\mathbb{E}_{X,\tilde{\theta}}[\tilde{Y}] < \mathbb{E}_{X,\hat{\theta}}[\hat{Y}]$
\label{prop1}
\end{proposition}

The proofs to all the propositions are in the Appendix.

\textbf{Remarks.}
Suppose $\mathsf{GD}$ converges to stationary points; in such a case $\tilde{\theta}$ described above compares  stationary points. In addition, if $f(\cdot,x)$ is convex in $\theta$ for all $x \in \mathcal{X}$ and a minimzer exists in the interior of $\Theta$, then all the initializers are equally good as they reach the same global minimum value. As a result, Assumption \ref{assm1} does not hold (since $\hat{Y}_0 = \hat{Y}_1$) and thus we cannot build a $\tilde{\theta}$ better than $\hat{\theta}$. Therefore, $\tilde{\theta}$ family described above is more suited for non-convex optimization problems.

Next, we consider a stricter version of Assumption \ref{assm1} to understand the best possible performance achievable by initializers from the family $\tilde{\theta} = (1-Z)\hat{\theta}_0 + Z\hat{\theta}_1$. We remove the expectation in Assumption \ref{assm1} and  require the probability of correct ordering to be larger than $\frac{1}{2}$.

\begin{assumption}(Probablistic ordering)
 $\exists\; \gamma \in (\frac{1}{2}, 1]$
\newline 
$\hat{Y}_0 < \hat{Y}_1 \implies p(\hat{\theta}_0, \hat{\theta}_1,X)\geq \gamma$, 
$\hat{Y}_0 > \hat{Y}_1 \implies 1-p(\hat{\theta}_0, \hat{\theta}_1,X)\geq \gamma$.
\label{assm2}
\end{assumption}

\begin{proposition} (Bounds on conditional random initialization)
$\mathrm{Assumption\; \ref{assm2}} \implies$
 \begin{equation*}
 \begin{split}
 & \mathbb{E}_{X,\hat{\theta}_0,\hat{\theta}_1}[\min\{\hat{Y}_0, \hat{Y}_1\}]\leq \mathbb{E}_{X,\tilde{\theta}}[\tilde{Y}]  \leq  \gamma \mathbb{E}_{X,\hat{\theta}_0,\hat{\theta}_1}[\min\{\hat{Y}_0, \hat{Y}_1\}] + (1-\gamma)\mathbb{E}_{X,\hat{\theta}_0,\hat{\theta}_1}[ \max\{\hat{Y}_{0}, \hat{Y}_1\}].
 \end{split}
 \end{equation*}
 \label{prop2}
\end{proposition}
In non-convex optimization, many solvers used in practice follow a multi-start based approach, i.e., do multiple random initializations and then use the best solution \cite{hu2009random,  marti2016multi}. The performance of a two-start based approach is $\mathbb{E}_{X,\hat{\theta}_0, \hat{\theta}_1}[\min\{\hat{Y}_0, \hat{Y}_1\}]$. From Proposition \ref{prop2}, if the initializer $\tilde{\theta}$ can order the two points sufficiently accurately, i.e. $\gamma$ is high, then the performance of initialization is close to the multi-start based approach. Next, we construct an initializer inspired from Proposition \ref{assm1} and \ref{assm2} that sequentially initializes $\mathsf{GD}$ to solve problem instances $\{X_i\}_{i=1}^{\infty}$.

\textbf{How can we construct initializers that are better than independent random initializer? }

\textbf{Vanilla approach.}  
 For an incoming instance $X$, sample $\hat{\theta}_0$ and $\hat{\theta}_1$ from $\mathbb{P}_{\hat{\theta}}$, compute the solutions $\hat{Y}_0$ and $\hat{Y}_1$ respectively using $\mathsf{GD}$ and select the solution with a smaller objective value.  Repeat this for first $N$  instances.  Train a model $\Psi$ that takes as input $X, \hat{\theta}_0, \hat{\theta}_1$ and outputs a real value  to predict  $\hat{Y}_1-\hat{Y}_0$. For every new $X$, sample $\hat{\theta}_0$ and $\hat{\theta}_1$. Define a Bernoulli  $Z_{\mathsf{vnila}}$ with $p_{\mathsf{vnila}}(\hat{\theta}_0, \hat{\theta}_1,X) = \frac{e^{\Psi(\hat{\theta}_0, \hat{\theta}_1,X)}}{1+e^{\Psi(\hat{\theta}_0, \hat{\theta}_1,X)}}$ as the probability $Z_{\mathsf{vnila}}=0$ and initialize from $\tilde{\theta}_{\mathsf{vnila}} = (1-Z_{\mathsf{vnila}}) \hat{\theta}_0 + Z_{\mathsf{vnila}} \hat{\theta}_1$. Define the corresponding solution value $\tilde{Y}_{\mathsf{vnilla}}= (1-Z_{\mathsf{vnila}}) \hat{Y}_0 + Z_{\mathsf{vnila}} \hat{Y}_1 $. The Pearson correlation between  random variables $U$, $V$ is $\rho(U,V)$.

\begin{proposition}
\label{prop3_van}
(Vanilla vs.\ independent random)
If $\rho(p_{\mathsf{vnila}}(\hat{\theta}_0, \hat{\theta}_1,X),\hat{Y}_1-\hat{Y}_0)>0$, then the Vanilla approach performs better than the independent random initializer, i.e.,  $\mathbb{E}_{X, \tilde{\theta}_{\mathsf{vnila}}}[\tilde{Y}_{\mathsf{vnila}} ]< \mathbb{E}_{X, \hat{\theta}}[\hat{Y} ]$.
\end{proposition}

Since the model $\Psi$ tries to predict $\hat{Y}_1 -\hat{Y}_0$, we expect a positive correlation between $\Psi$ and $\hat{Y}_1 -\hat{Y}_0$. As  $p_{\mathsf{vnila}}$ is an increasing invertible transformation of $\Psi$, we  expect a positive correlation between $p_{\mathsf{vnila}}$ and $\hat{Y}_1 -\hat{Y}_0$ (see Appendix for details). Proposition \ref{prop3_van} is reminiscent of the following result in binary classification: as long as a classifier's output is positively correlated with the binary label, the classifier performs better than a uniformly random classifier (see Appendix for details). The Vanilla approach only works with two initial values and requires access to the outcomes from multiple initializations for the same problem instance. Next, we propose a more general approach next that works with multiple initial values and also does not require access to multiple outcomes for the same instance.

\subsection{Val-Init: Learn from solution values } 
The approach is divided into two phases. In the first phase comprising $N$ instances, we find the solution for the $i^{th}$ instance $X_i$ using random initialization  $\hat{\theta} \sim \mathbb{P}_{\hat{\theta}}$. In the second phase, we learn a model $h_{\mathsf{val}}: \mathbb{R}^{m } \times \mathbb{R}^{n} \rightarrow \mathbb{R}$ to predict for initialization $\hat{\theta}$ and instance $X_i$   the  objective value at solution  $g(\hat{\theta},X_i)$.  We use the learned model $h_{\mathsf{val}}$ to select initializations as follows. For each new instance, generate $M$ i.i.d.\ random initializations $\{\hat{\theta}_k\}_{k=0}^{M-1}$ from $\mathbb{P}_{\hat{\theta}}$ and compute the predicted value of the final solution using $h_{\mathsf{val}}$. Select the initialization with lowest predicted value.  We provide the algorithmic description in Algorithm \ref{alg:VIA}; we refer to the approach as Val-Init  as it uses the predicted values to initialize. We define some random variables to describe the Val-Init initializer. Define $\{Z_{k}\}_{k=0}^{M-1}$ as $M$ random variables, where each $Z_k \in \{0,1\}$ and $\sum_{k=0}^{M-1} Z_{k}=1$. If $\arg\min \{h_{\mathsf{val}}(X,\hat{\theta}_{k})\}_{k=0}^{M-1} =j$, then $Z_{j}=1$, else $Z_{j}=0$, thus $Z_j$ indicates the initializer  selected. Define the Val-Init initializer as $\tilde{\theta}_{\mathsf{val}} = \sum_{k=0}^{M-1}Z_k\hat{\theta}_k$ and the corresponding solution value as $\tilde{Y}_{\mathsf{val}} = \sum_{k=0}^{M-1}Z_k\hat{Y}_k$, where $\hat{Y}_k = g(X, \hat{\theta}_k)$ is the solution value  achieved from $\hat{\theta}_k$.
\begin{proposition} \label{prop4} (Val-Init vs.\ independent random)
If   $\rho(Z_{k}, \hat{Y}_{M-1}-\hat{Y}_{k})>0$  $\forall k\in\{0,\cdots, M-2\}$, then Val-Init performs better  than the independent random initializer, i.e., $\mathbb{E}_{X, \tilde{\theta}_{\mathsf{val}}}[\tilde{Y}_{\mathsf{val}} ]< \mathbb{E}_{X, \hat{\theta}}[\hat{Y} ]$. 
\end{proposition}

From the above Proposition it follows that if the objective value  $\hat{Y}_k$ is small, then we expect the corresponding selection variable $Z_k$ to be large.

\begin{algorithm}[tb]
   \caption{Val-Init and Arg-Init Methods}
\label{alg:VIA}
\begin{algorithmic}
   \STATE {\bfseries Input:} Instances $\{X_i\}_{i=1}^{\infty}\sim \mathbb{P}_X$, $\mathcal{H}_{\mathsf{val}}$, $\mathcal{H}_{\mathsf{arg}}$ Hypothesis classes for $h_{\mathsf{val}}$,  $h_{\mathsf{arg}}$ respectively
\FOR{$i \in \{1,..\infty\}$}
\IF{$i\leq N$}
\STATE $\hat{\theta}^{i} \sim \mathbb{P}_{\hat{\theta}}$,$f_{i} = f(\cdot,X_i)$, $\hat{\theta}^{\dagger}_{i}, \hat{Y}_i \xleftarrow{\mathsf{Algthm} \ref{alg1:GDA}} \mathsf{GD}(f_i, \hat{\theta}^{i})$, 
\ENDIF
\STATE $h_{\mathsf{val}} = \arg\min_{h \in \mathcal{H}_{\mathsf{val}}} \sum_{j=1}^{N}\ell\big(h(\hat{\theta}^{j}, X_j), \hat{Y}_j\big)$ 
\STATE $h_{\mathsf{arg}} = \arg\min_{h \in \mathcal{H}_{\mathsf{arg}}} \sum_{j=1}^{N}\ell\big(h(\hat{\theta}^{j}, X_j), \hat{\theta}^{\dagger}_{j}\big)$ 
\IF{$i \geq  N+1$}
\STATE $\{\hat{\theta}_{m}^{i}\sim \mathbb{P}_{\hat{\theta}}\}_{m=0}^{M-1}$,
\STATE $m^{*} = \arg\min \{h_{\mathsf{val}}(\hat{\theta}_{m}^{i}, X_i)\}_{m=0}^{M-1}$, \STATE $\tilde{\theta}_{\mathsf{val}} = \hat{\theta}_{m^{*}}^{i}$,$f_{i} = f(\cdot,X_i)$,
\STATE $\tilde{\theta}_{\mathsf{val}}, \tilde{Y}_{\mathsf{val}} = \mathsf{GD}(f_i, \hat{\theta}_{m^{*}}^{i})$ 
\STATE $\tilde{\theta}_{\mathsf{arg}}, \tilde{Y}_{\mathsf{arg}} = \mathsf{GD}(f_i, h_{\mathsf{arg}}(\hat{\theta}_0^{i},X_i))$
\ENDIF
\ENDFOR

\end{algorithmic}
\end{algorithm} 
\textbf{How does Val-Init compare to $M$ multi-starts?}

Define the set of objective function values achievable from all possible initializations for an instance $x$ as $\mathcal{F}(x)$.
For each problem instance $x$, define a mapping $\Phi_{x}$ that maps a real input value $y$ to the set of initializations from which $y$ can be achieved using $\mathsf{GD}$.  
\begin{assumption} (Minimum separation)
  For an $x\in \mathcal{X}$, if $y \in \mathcal{F}(x)$, $z \in \mathcal{F}(x)$, $y\not=z$, then $\exists\; \Delta>0$ s.t. $|y-z|\geq \Delta$.
\label{assm3}
\end{assumption}

If $\mathsf{GD}$ arrives at a finite set of distinct values from all initializations, then the above assumption holds.
\begin{assumption} (Bounded error)
Consider a problem instance $x\in \mathcal{X}$. For each $y$ which satisfies $\mathbb{P}_{\hat{\theta}}\Big(\hat{\theta}\in \Phi_{x}(y)\Big)>0$,  $\mathbb{E}_{\hat{\theta}}\Big[|y- h_{\mathsf{val}}(\hat{\theta},x)|^2 \Big| \hat{\theta} \in \Phi_{x}(y), X=x\Big] \leq \frac{\Delta^2 \zeta}{4}$,
where $0<\zeta < \frac{\tilde{\epsilon}}{Mf^{\mathsf{sup}}}$ and $f^{\mathsf{sup}}=\sup_{\theta \in \Theta}f(\theta,x)<\infty$  and $\tilde{\epsilon}>0$ is a small positive quantity. 
\label{assm4}
\end{assumption}

 Assumption \ref{assm4} requires the prediction error of $h_{\mathsf{val}}$ to be bounded. The bound is not restrictive, especially if $\Delta$ (the difference between  distinct objective values at solution) is not small.

\begin{proposition} (Val-Init vs.\ $M$ multi-starts) For problem instance $x$, if Assumptions \ref{assm3} and \ref{assm4} hold, then Val-Init is an $\tilde{\epsilon}$-approximation of the multi-start approach:
\begin{equation*}
\begin{split}
&\mathbb{E}_{ \{\hat{\theta}_k\}_{k=0}^{M-1}}[\tilde{Y}_{\mathsf{val}}|X=x] \leq \mathbb{E}_{\{\hat{\theta}_k\}_{k=0}^{M-1}}[\min_{k \in \{0,..,M-1\}}\{\hat{Y}_{k}\} |X=x]+ \tilde{\epsilon}.
\end{split}
\end{equation*}
\label{prop5}
\end{proposition}

From the above Proposition, Val-Init can perform close to the multi-start based approach without running $\mathsf{GD}$ multiple times from different start points for same problem instance. Note that Val-Init learns from the solution's function values. Next, we propose an approach that learns from the solution's argument values.

\subsection{Arg-Init: Learn from solution arguments} 
 This approach also consists of two phases. The first phase of this approach is identical to the first phase of Val-Init: we collect data using random initializations. In the second phase, we learn a model $h_{\mathsf{arg}}: \mathbb{R}^{m } \times \mathbb{R}^{n} \rightarrow \mathbb{R}^{m}$ that maps  the initialization $\hat{\theta}$ and the instance $X$ to predict the solution argument $g^{\dagger}(\hat{\theta},X)$ (output of $\mathsf{GD}$ defined in Section \ref{secn: prob_formln}). For each new instance $X$, sample $\hat{\theta} \sim \mathbb{P}_{\hat{\theta}}$ and  input it to the learned model $h_{\mathsf{arg}}$ to generate the initialization point $h_{\mathsf{arg}}(\hat{\theta},X)$. We provide the algorithmic description in Algorithm \ref{alg:VIA}; we refer to this approach as Arg-Init as it uses solution arguments to initialize.  Arg-Init predicts an initialization, which it hopes is closer to the solution of a random initialization $\hat{\theta}$  as that can lead to faster convergence.  Define this condition as an event  called ``$\eta$-factor reduction'' given as  $\|h_{\mathsf{arg}}(\hat{\theta},X) - g^{\dagger}(\hat{\theta}, X)\| \leq \eta
\|\hat{\theta} - g^{\dagger}(\hat{\theta}, X)\|$, where $\eta<1$.  See Appendix for how this event leads to $\eta$-factor reduction in the number of iterations to convergence.

\begin{assumption}  (Uniform continuity of $\mathbb{P}_{\hat{\theta}}$)  For each $\epsilon>0$, $\exists$  a $\delta$ such that for any $\theta$ in the interior of $\Theta$ and a ball $B_{\delta}= \{\bar{\theta} \in \Theta \;|\; \|\bar{\theta} -\theta\| \leq \delta\}$ of radius $\delta$ around $\theta^{\dagger}$ the probability  $\mathbb{P}_{\hat{\theta}}(\hat{\theta}  \in B_{\delta}) \leq \epsilon$.
\label{assm5}
\end{assumption} 

\begin{proposition} \label{prop6} (Arg-Init vs.\ independent random) If Assumption \ref{assm5} holds and the prediction error of $h_{\mathsf{arg}}$ is small, i.e.  $0<s<1$, $0<\eta<1$, 
$\mathbb{E}_{ X,\hat{\theta}}\big[\|h_{\mathsf{arg}}(\hat{\theta},X) - g^{\dagger}(\hat{\theta}, X)\|^2\big] \leq s (1-\epsilon)(\delta \eta)^2$, where $\epsilon$, $\delta$ are from Assumption \ref{assm5}, then with probability $(1-s)(1-\epsilon)$  Arg-Init  achieves $\eta$-factor reduction.
\end{proposition}

 If $\epsilon$ decreases, then $\delta$ in Assumption \ref{assm5} decreases, the upper bound on the error in the Proposition \ref{prop6} should decrease and the probability of $\eta$-factor reduction increases. 
 Before we proceed to the experiments, we close with a comparison of Arg-Init and Val-Init in Figure \ref{fig_arg_val}. We expect Arg-Init to work in both convex and non-convex problems. Since Val-Init is based on comparing local minima (stationary points) it is not suited for convex problems. Val-Init it can offer advantage over Arg-Init whenever the local minima values differ substantially (as shown in Figure \ref{fig_arg_val} and Ackley function experiment Section \ref{secn: expmts}).  From a learning point of view Val-Init \ref{alg:VIA} needs to learn a function with scalar output, while Arg-Init needs to learn a function with output dimesions same as the optimization variable $\theta$.

\begin{figure}[ht]
\begin{center}
\centerline{\includegraphics[trim=0.5cm 2.1cm 1.5cm 2cm,clip=true,width=3in]{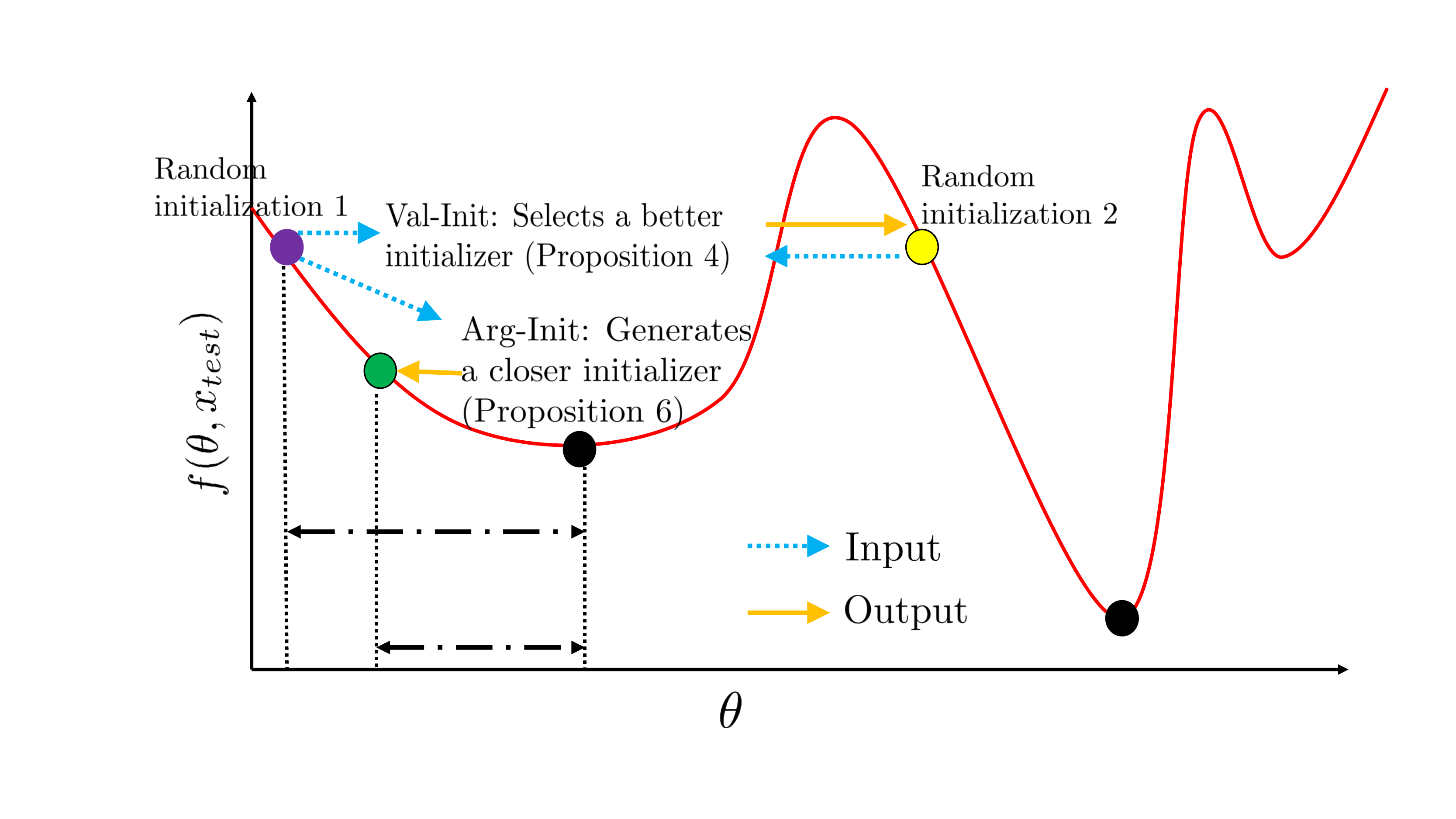}}
\caption{Arg-Init vs.\ Val-Init overview}
\label{fig_arg_val}
\end{center}
\end{figure}

\section{Experiments}
\label{secn: expmts}

In this section, we apply our methods (Val-Init and Arg-Init) to solve various optimization problems. We first illustrate our methods on a a standard non-convex optimization problem. (We also provide synthetic convex optimization experiments in the Appendix.) We then illustrate our methods on real applications such as generating adversarial examples, generating contrastive explanations, and sum-rate optimization. We compare our initialization methods with the random initialization, zero initialization, and initializers learned using a MAML-based approach \cite{finn2017model}. For all methods, the choice of architectures, hyperparameters, train and test sizes, validation error  of $h_{\mathsf{val}}$ and $h_{\mathsf{arg}}$ (during training) are in the Appendix. We use all the initialization methods  as single-start methods (comparing single-start with multi-start is not fair as multi-start method has a higher computational cost). For each problem family (generate explanations for a model), we are provided different problem instances (data points) and we need to generate solutions to them (explanations). For each problem instance, we initialize gradient descent using different initialization methods to search the solutions. The performance is defined as the average of the objective function values (at the solution) across different problem instances. We compare the methods in terms of the performance achieved  vs.\ the number of iterations of gradient descent. Histograms showing the distribution of performance across problem instances is in the Appendix.

\begin{figure*}[htbp]
\centering
  \includegraphics[width=0.4\textwidth]{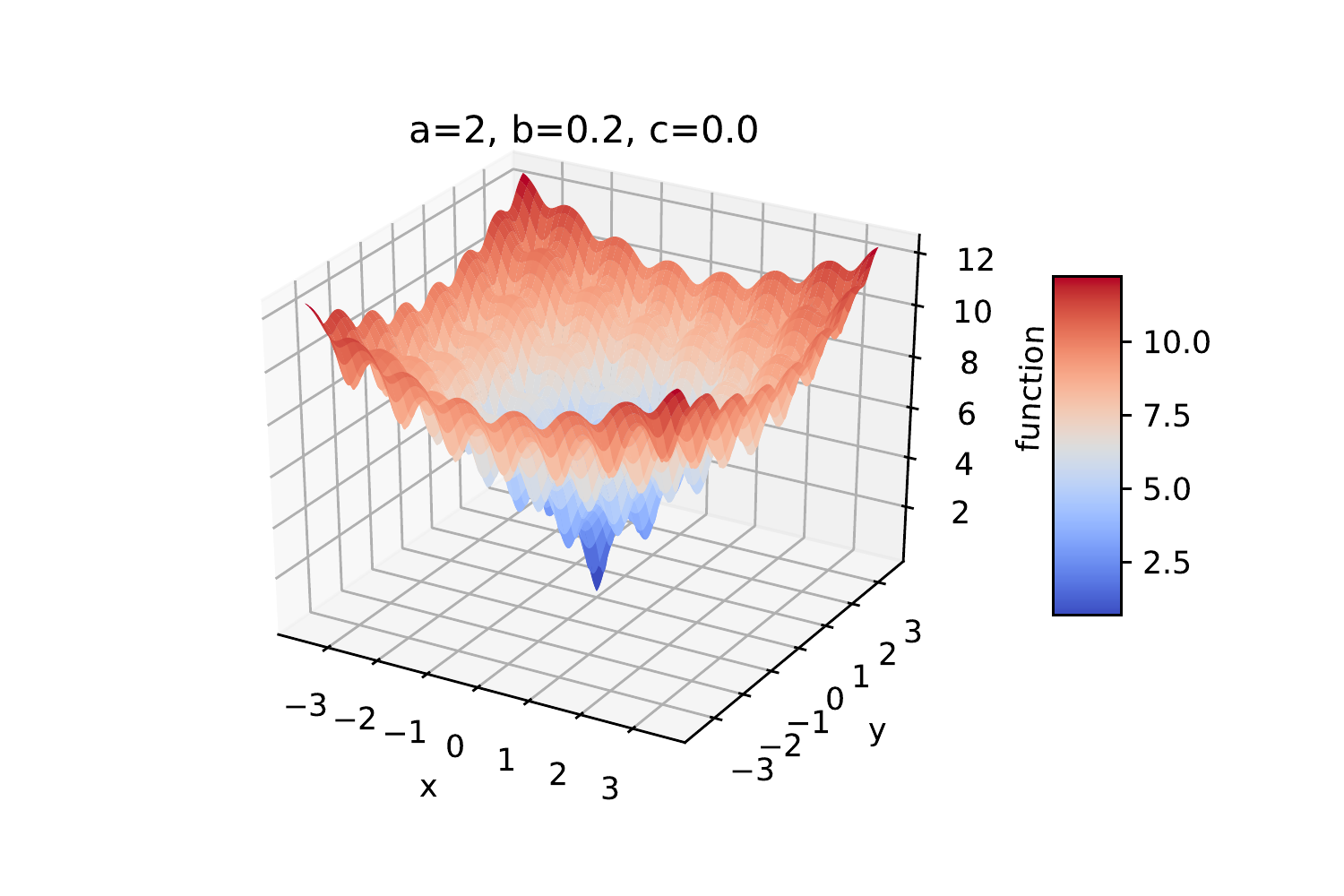}
  \includegraphics[width=0.4\textwidth]{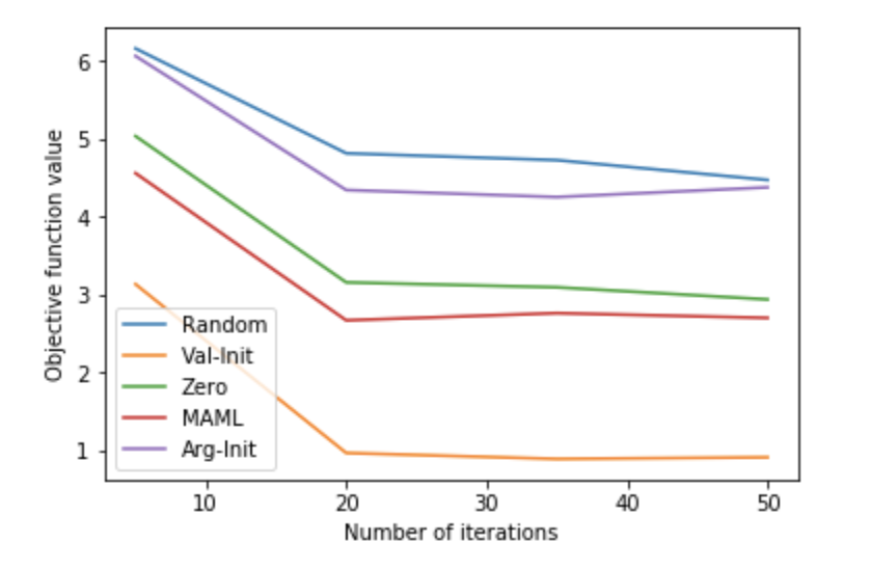}
\caption{Illustration of Ackley function (left), and objective function value vs.\ number of iterations (right)}
\label{fig1}
\end{figure*}

% \begin{figure*}[htbp]
%   \centering
% %  \begin{minipage}{.5\textwidth}
% %   \centering
%   \includegraphics[width=.4\textwidth]{Ackley_june3.pdf}
% %   \label{fig:test1}
% %  \end{minipage}%

% %  \begin{minipage}{.5\textwidth}
% %   \centering
%   \includegraphics[width=.4\textwidth]{ackley_obj_val.png}
% %   \label{fig:test2}
% %  \end{minipage}

% \caption{Illustration of Ackley function (left), and objective function value vs.\ number of iterations (right)}
% \end{figure*}

\begin{figure}[ht]
\begin{center}
\centerline{\includegraphics[trim=0cm 0cm 1.5cm 0.5cm,clip=true,width=2.5in]{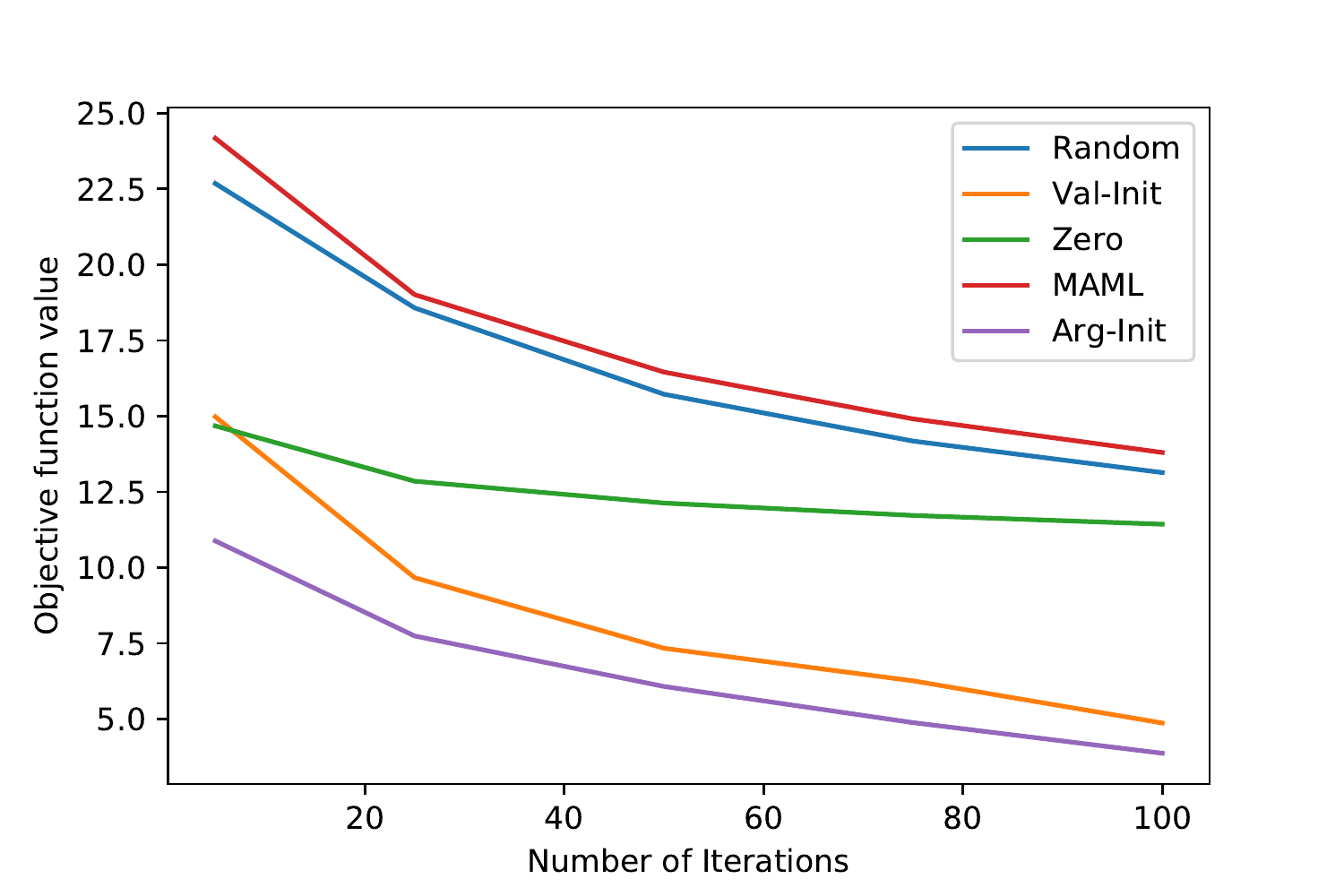}}
\caption{Adversarial examples on CIFAR-10: Objective function value vs.\ number of iterations}
\label{fig0}
\end{center}
\end{figure}

\subsection{Ackley function minimization} 
Ackley functions are used to study non-convex optimization problems \cite{ackley2012connectionist}. Ackley function is defined as $\mathsf{A(x,y,a,b,c)}= -\mathsf{a exp(-b\frac{\sqrt{(x-c)^2+(y-c)^2}}{2})}  -\mathsf{exp(\frac{cos(2\pi(x-c)) + cos(2\pi(y-c))}{2}) +exp(1) +a}  $. The global minimum of the function is at $\mathsf{(c,c)}$ and the minimum value is zero. In Figure \ref{fig1}, we show an instantiation of Ackley function. The function has many minima that are worse than the global minimum. We are given a set of Ackley functions with parameters $\mathsf{a,b,c}$ drawn i.i.d. from $\mathsf{a\sim 20 + U[0,10]}$, $\mathsf{b\sim 0.2 + U[0,0.1]}$, and $\mathsf{c\sim U[0,2]}$, where $\mathsf{U}$ is a uniform distribution. We want to learn an initializer that takes $\mathsf{a,b,c}$ as input and generates an initialization. We train Arg-Init, Val-Init (Algorithm \ref{alg:VIA}) using 1500 i.i.d draws of  $\mathsf{a,b,c}$ from the above uniform distributions. In Figure 2, we compare the  methods in terms of the average objective function value vs.\ iterations in gradient descent, where the average is taken over 500 unseen  instances $\mathsf{a,b,c}$. Val-Init  performs better across different iteration values. Arg-Init did not generate a good initialization in this case suggesting Val-Init is potentially more useful in highly non-convex problems.

\subsection{Datasets, models for adversarial examples and contrastive explanations}

$\;\;$ \textbf{CIFAR-10.} We trained a convolutional neural network (on the default train-test split of CIFAR-10) that achieves  70\% accuracy in predicting the classes.

 \textbf{MNIST.} We trained a 2 layer neural network (on the default train-test split of MNIST digits) that achieves  97\% accuracy in predicting the digits.

 \textbf{Waveform.}  The dataset consists of 40 attributes that help identify the wave type (from three types of waves). We divided the data into 75\% training and 25\% testing.  We trained a 2 layer neural network that achieves 85\% accuracy to predict the wave type. 
 
\textbf{HELOC.} Home equity line of credit (HELOC) data contains the information about the applicant in their credit report and their loan repayment history. We  build a model to predict whether applicants will make timely payments. We divided the data into 75\% training and  25\% testing. We trained a 2 layer neural network using the training data that achieves 72\% accuracy in predicting repayments. 

For each model trained above and the respective train split, we  run gradient descent  with random initialization to generate adversarial example (contrastive expalanation) for each point in the train split  and learn $h_{\mathsf{val}}$ and $h_{\mathsf{arg}}$. Next, for each point in the respective test split we run gradient descent with different initializers  to generate adversarial example (contrastive explanations). We compare the methods in terms of the average quality of the adversarial examples (explanations)  across the test instances. Details on the objective used for gradient descent are provided next.

\begin{table}[t]
\caption{\textbf{Adversarial examples.} Comparing performance in terms of the penalty-based objective. Average distance value, the fraction of instances when constraints are not satisfied are mentioned below the penalty-based objective value. The best methods in terms of the penalty-based objective are in bold. 
All methods are run for 100 iterations in Waveform and MNIST and 10 iterations in HELOC. 
}

\begin{center}
\begin{small}
\begin{sc}
\begin{tabular}{lcccr}
\toprule
Method & HELOC & Waveform & MNIST   \\
\midrule
               Arg-Init & \textbf{-0.37} & \textbf{-0.89} & 1.66 \\
               &(0.44, 0.016)&(0.45, 0.03)& (1.15, 0.17) \\
                &&&\\
                          Val-Init & 1.39 & -0.67 & \textbf{0.98}\\
          &(1.85, 0.25)&(0.85, 0.22)& (0.81, 0.00) \\
          &&&\\
         Random & 1.42 & -0.66& 3.41  \\ 
         & (1.86, 0.26) & (0.85, 0.21) & (1.90, 0.00) \\ 
         &&&\\
                  Zero     & 0.021 & \textbf{-0.89} & 4.60 \\
                  &(0.50, 0.013)&(0.32, 0.09)& (0.97, 0.02) \\ 
                     &&&\\
                           MAML     & -0.017 & \textbf{-0.89} & 7.09 \\ 
                           &(0.94, 0.00)&(0.69, 0.00)& (3.76, 0.00) \\ 
\bottomrule
\end{tabular}
\end{sc}
\end{small}
\end{center}
\label{table3}
\end{table}

\begin{table}[t]
\caption{\textbf{Contrastive examples.} Comparing performance in terms of the penalty-based objective defined in \cite{dhurandhar2018explanations}. Average distance value, the fraction of instances when constraints are not satisfied are mentioned below the objective value. The best methods in terms of the optimization objective are in bold.
All methods are run for 100 iterations in Waveform and MNIST and 10 iterations in HELOC.
}
\begin{center}
\begin{small}
\begin{sc}
\begin{tabular}{lccr}
\toprule
Method & Waveform & MNIST   \\
\midrule
         Arg-Init &  \textbf{0.40}& \textbf{-0.03} \\ 
         &(0.42, 0.19)&(0.93, 0.02)  \\
         && \\ 
                 Val-Init &  2.23&  0.54  \\
                 &(1.95, 0.23)& (2.16, 0.01)\\ 
         && \\ 
         Random &  3.51& 1.49  \\ 
         &(1.95, 0.44) &  (3.90, 0.00) \\ 
         && \\ 
         Zero     & 0.61 & 4.29 \\
         & (0.36, 0.23) & (0.95, 0.69) \\ 
         && \\ 
         MAML & \textbf{0.40}     & 2.32 \\ 
         & (0.41, 0.19) & (4.56, 0.00)\\ 
\bottomrule
\end{tabular}
\end{sc}
\end{small}
\end{center}
\label{table3_n}
\end{table}

\subsection{Adversarial examples}
 Recall, if $\beta=0$ and  $\Theta=\mathbb{R}^{m}$ in equation \eqref{eqn: stat_opt_example1}, we obtain the problem of finding an adversarial example for a binary classifier.  We then reformulate equation \eqref{eqn: stat_opt_example1} into a commonly used penalty-based objective  as follows $\|\theta\| + \mathcal{L}\big(y(u(x+\theta))\big)$ ($\mathcal{L}$ is from equation 3 in \cite{cheng2018query}). We compare all the methods in terms of this penalized objective function as it reflects the quality of the adversarial example. 
In Figure \ref{fig0}, we compare the performance of the methods (penalized objective averaged across test instances in the data) vs.\ the number of iterations in gradient descent. Arg-Init is the best closely followed by Val-Init. Except for zero initialization, every method was able to generate an example from a different class than the data instance, i.e., satisfy the constraint in equation \eqref{eqn: stat_opt_example1}, for each input instance.  Before moving to next comparisons, we contrast the validation mean square error at the end of first  last epoch in training to show that even in a relatively more complex setting of CIFAR-10 it is possible to learn $h_{\mathsf{val}}$ (1 million parameters) and $h_{\mathsf{arg}}$ (1.5 million parameters). We find that validation  MSE values decrease from 47.68 to 24.92 and from 0.0244 to 0.0171 for $h_{\mathsf{val}}$ and $h_{\mathsf{arg}}$ respectively. We provide these values for all the other experiments in the Appendix.  For other datasets we provide a table of comparisons (Table \ref{table3}) for a fixed number of iterations and the plots for objective vs.\ iterations are given in the Appendix. In Table \ref{table3}, we see that Val-Init is the best  on MNIST, Arg-Init is the best on HELOC, and Arg-Init is tied  with Zero and MAML on Waveform.

 In Figure 4, we contrast our approach to random initialization for adversarial examples on MNIST. The final solutions shown all satisfy the constraint that they are classified from a different class than the original image's class 1. The solution generated from our approach is visually indistinguishable from class 1, while that is not the case for the solution generated from random approach. We found this to be the case across most examples.  In Figure 4, we also do a visual contrast of the initial values used by Val-Init, Arg-Init, and random initialization. Our approaches finds points that  morph easily into  class 1 (class for which adversarial examples are generated).  
  
  \begin{figure*}
\begin{minipage}{.5\textwidth}
  \centering
  \includegraphics[trim=1.5cm 0.7cm 1.5cm 1.5cm, width=2.5in]{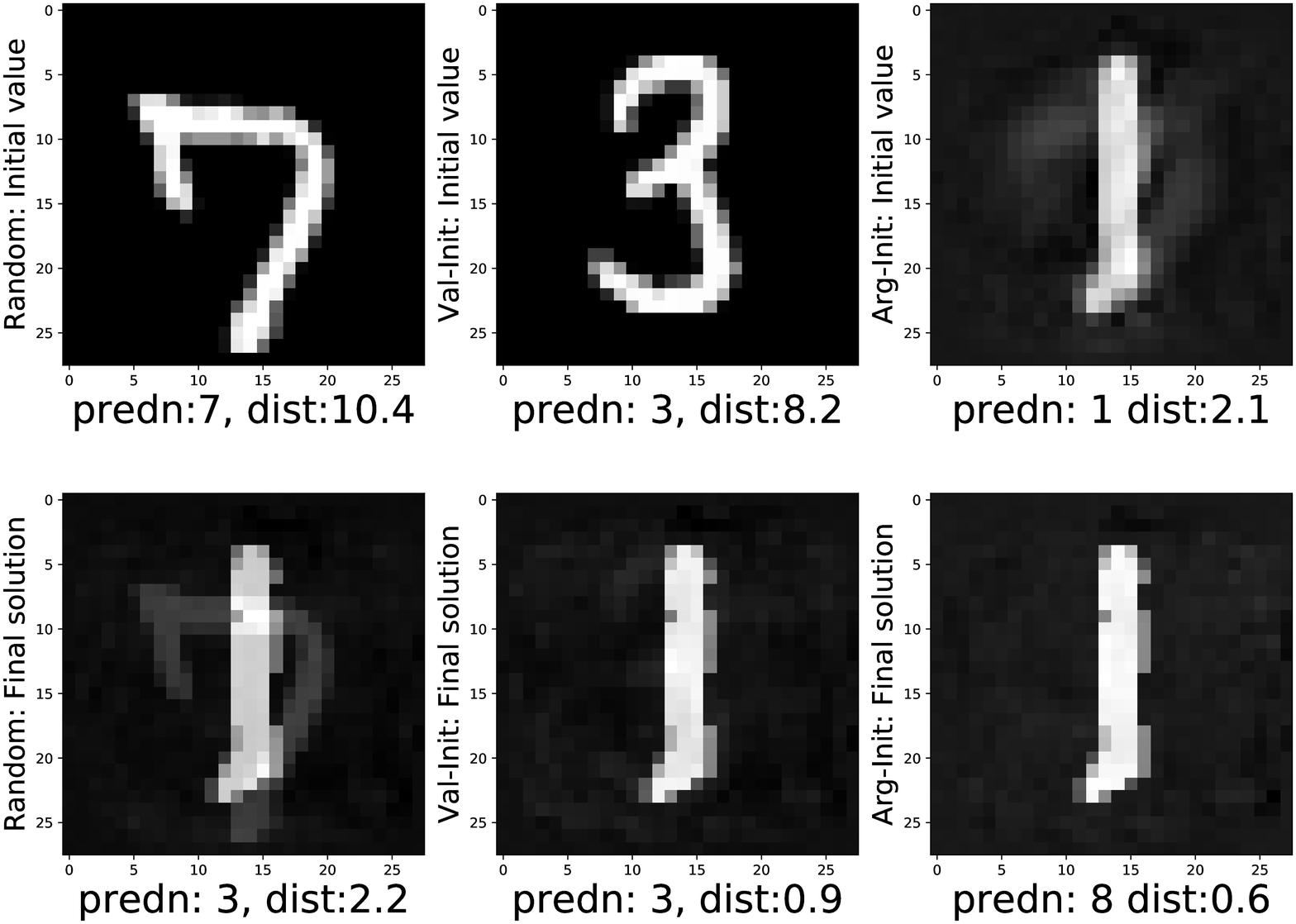}
  \text{a. Adversarial examples}
\end{minipage} 
\begin{minipage}{.5\textwidth}
  \centering
  \includegraphics[trim=1.5cm 0.7cm 1.5cm 1.5cm, width=2.5in]{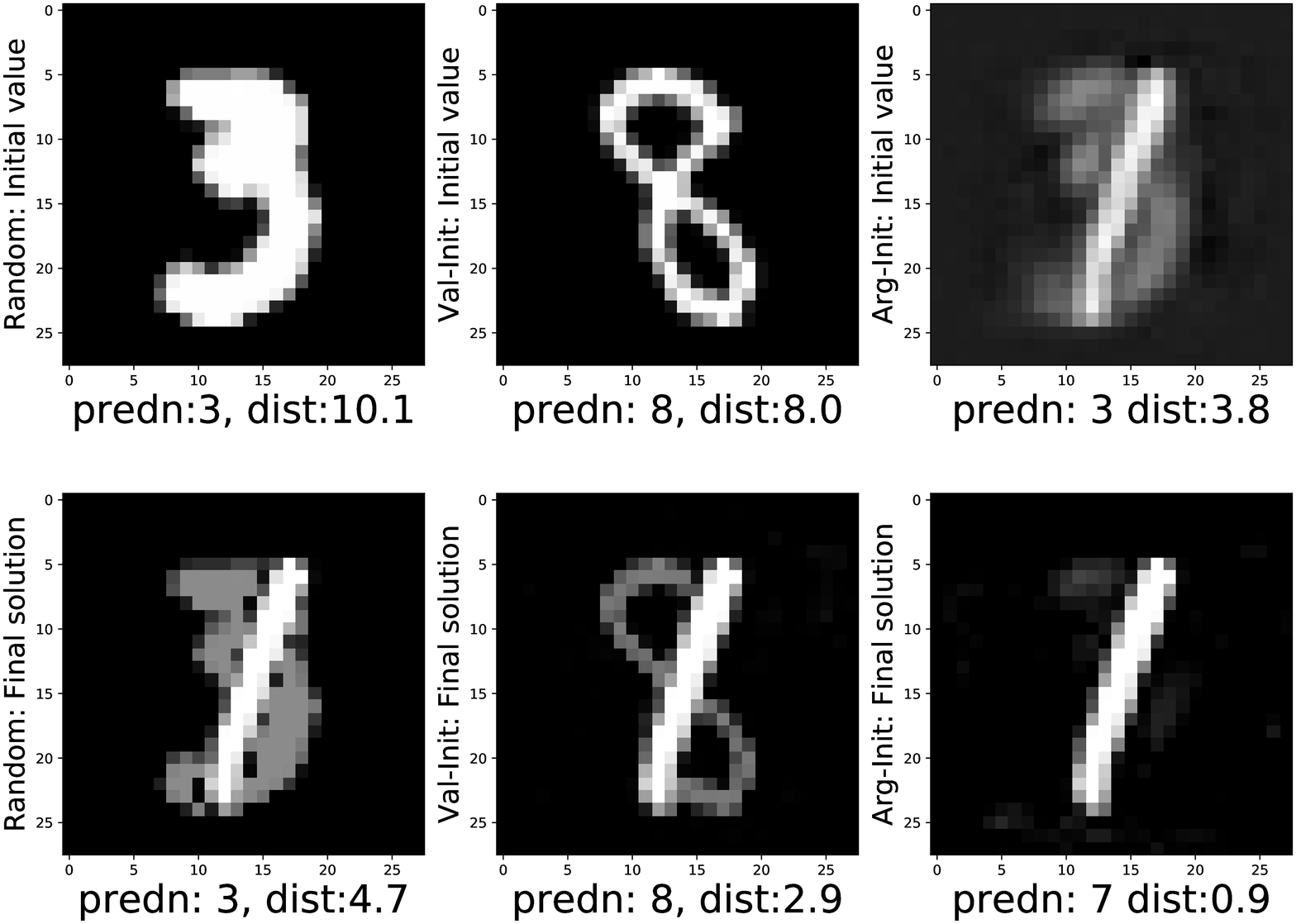}
  \text{b. Contrastive explanations}
\end{minipage}
  \label{fig:test2_n}
\caption{Comparing  solution quality for different initializations  for same iterations. Top row: initializations, bottom row: respective final solns. Left column (random), middle column (Val-Init), right column (Arg-Init).  Below each image we state predicted class \& distance from original instance. }
\end{figure*}

\begin{figure}[ht]
\begin{center}
\centerline{\includegraphics[trim=1cm 0.7cm 1.5cm 1cm,clip=true,width=2.5in]{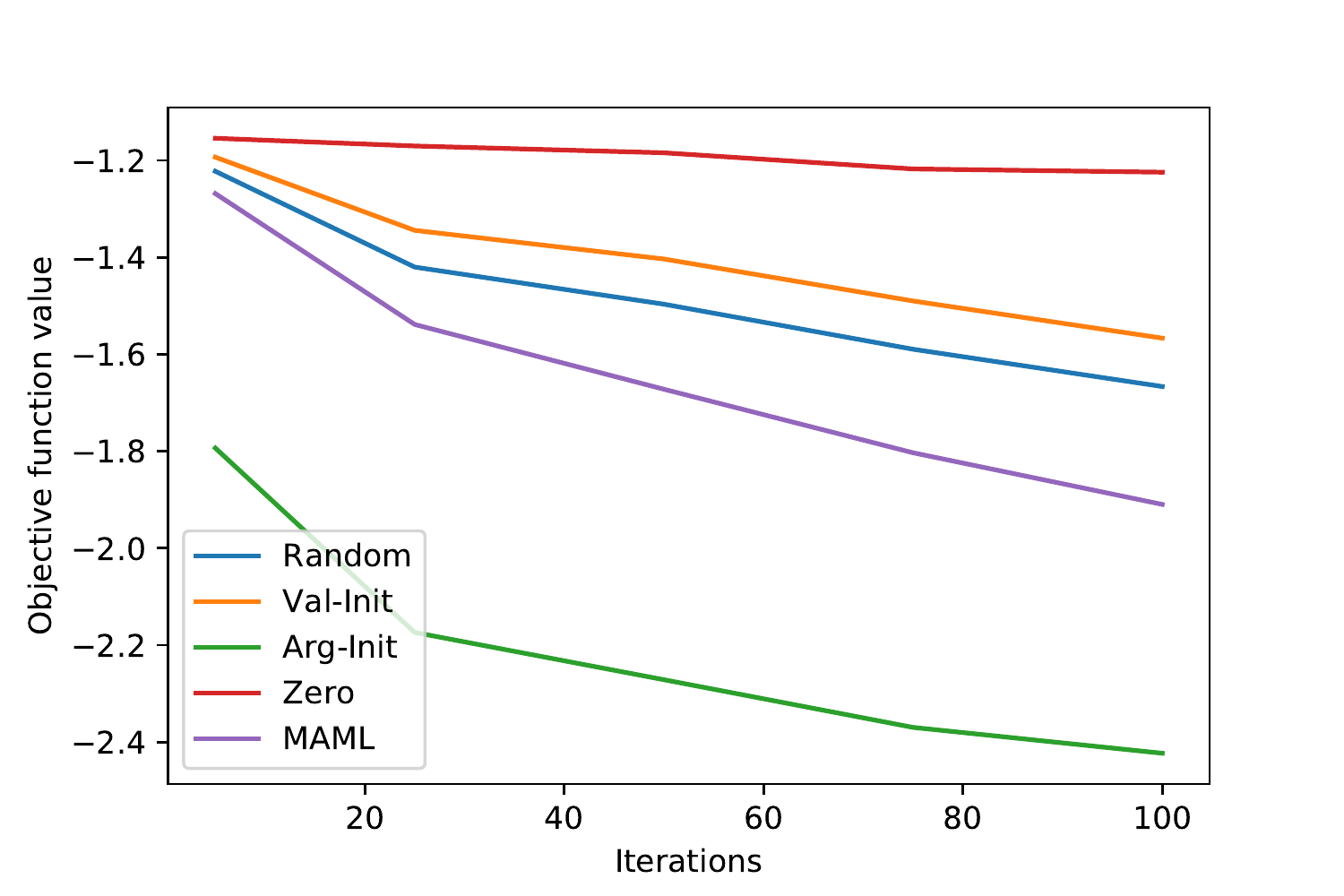}}
\caption{ $-$Sum rate vs.\ number of iterations}
\label{fig1_n1}
\end{center}
\end{figure}

\subsection{Contrastive explanations (PNs)}

Recall, if $\beta>0$ and  $\Theta$ is constrained to PNs  \cite{dhurandhar2018explanations} in equation \eqref{eqn: stat_opt_example1}, we obtain the problem of finding PNs based explanations. We then reformulate equation \eqref{eqn: stat_opt_example1} into a penalty-based objective minimization in \cite{dhurandhar2018explanations}.  We compare the methods in terms of the average penalty-based objective. No PNs were found in the HELOC data and the PN framework does not apply to colored images as in CIFAR-10, thus no comparisons were possible on HELOC and CIFAR-10 dataset.   In Table \ref{table3_n}, we summarize the comparisons. On MNIST digits, Arg-Init is the best and Val-Init also does well in generating PNs. On Waveform,  Arg-Init and MAML perform the best.  In Figure 4, we contrast our approaches to random initialization for PNs on MNIST. The final solutions shown satisfy the constraint that they are classified from a different class than the original image's class 1. The solution generated from our approach are sparser and thus better PNs; this is true across many examples.  In Figure 4, we also show the initial values used. Our approaches  start from points that  easily morph into sparse PNs.

\subsection{Sum-rate optimization}

We compare the different approaches to initialize gradient descent to  solve the sum-rate optimization problem (described in Section 2). We train Val-Init and Arg-Init using 5000 problem instances (solved using gradient descent with random initialization) for 15 users. Each problem instance is a new independent draw of the channel matrix from a standard uniform random distribution (each element in the matrix is i.i.d. drawn from uniform distribution over [0,10]).  We test on the 500 instances from the same distribution; see Figure \ref{fig1_n1} for comparisons. Arg-Init beats all methods (25 \% gain).

\subsection{Overall comparison summary}

Let us start with Arg-Init approach. In seven out of the eight instances, it is at least as good (two times) or better (five times) than all the existing methods. Next, consider the Val-Init approach.  It performs better than the existing methods in four out of eight instances. Recall the Val-Init is not designed for problems where the local minimum values do not differ by a lot for e.g., convex optimization problems, where all the minima take the same value. In the instances where Val-Init did not outperform others --  tabular datasets for adversarial examples and sum-rate optimization -- it is possible that the different local-minimum values do not differ by a lot thus not giving much room for Val-Init method to perform well.  In contrast, see the example of Ackley function and relatively more complex datasets such as  CIFAR-10 and MNIST, Val-Init was able to exploit the variance in different local minimum values.

\section{Conclusion}

In optimization solvers that use gradient descent, initialization is crucial to finding ``good'' local minima. Therefore, it is common practice to hand-engineer initialization rules. In this work, we developed theory that led to two approaches that learn how to initialize for an input problem class. The first approach learns to discriminate different initial values and selects the best, and the second uses an initial value to create another value that is closer to the local minima. We carried out experiments on several datasets (real and synthetic) and on a variety of optimization problems (convex and non-convex) where at least one of these simple approaches has state-of-the-art performance. In the future, it would be interesting to extend these approaches to more general settings (viz. initializing NNs) with the possibility of even intelligently combining them. 

\section{Appendix}
We organize the Appendix as follows.  In the first section, we provide the proofs to all the propositions. In the second section, we provide supplementary materials for the experiments.

% We want to show that $\frac{e^m}{1+e^{m}}$ is also positively correlated with $Y_{}$
\subsection{Proofs}

In this section, we provide the proofs to the propositions. We restate the propositions for reader's convenience. 
We restate Proposition \ref{prop1} below.
\begin{proposition}
Assumption \ref{assm1} $\iff$ $\mathbb{E}_{X,\tilde{\theta}}[\tilde{Y}] < \mathbb{E}_{X,\hat{\theta}}[\hat{Y}]$
\label{prop1_append}
\end{proposition}

% \end{theorem}
\begin{proof}
%  $\hat{f}= \mathbb{E}[g(\hat{\theta},X)]$ and $\tilde{f} = \mathbb{E}[g(\tilde{\theta},X)]$. Simplify $\tilde{f}$ to obtain

Let us first simplify  $\mathbb{E}_{X,\tilde{\theta}}[\tilde{Y}]$
% $\mathbb{E}[g(\tilde{\theta},X)]$. 

\begin{equation}
\begin{split}
   \mathbb{E}_{X,\tilde{\theta}}[\tilde{Y}] & =    \mathbb{E}_{X,\tilde{\theta}}[g(\tilde{\theta},X)] \\ 
   & =   \mathbb{E}_{X,\tilde{\theta}}[g((1-Z)\hat{\theta}_0 + Z\hat{\theta}_1,X)] \\ 
   & = \mathbb{E}_{X,\tilde{\theta}}[(1-Z)g(\hat{\theta}_0,X)   + Zg(\hat{\theta}_1,X)] \\
   & = \mathbb{E}_{X,\hat{\theta}_0, \hat{\theta}_1, Z}[(1-Z)g(\hat{\theta}_0,X)   + Zg(\hat{\theta}_1,X)] \\ 
   & = \mathbb{E}_{X,\hat{\theta}_0, \hat{\theta}_1}[(1-\mathbb{E}_{Z|X,\hat{\theta}_1, \hat{\theta}_2}[Z])g(\hat{\theta}_0,X)   + \mathbb{E}_{Z|X,\hat{\theta}_0, \hat{\theta}_1}[Z]g(\hat{\theta}_1,X)]  \\
   & = \mathbb{E}_{X,\hat{\theta}_0, \hat{\theta}_1}[p(\hat{\theta}_0, \hat{\theta}_1,X)g(\hat{\theta}_0,X)   + (1-p(\hat{\theta}_0, \hat{\theta}_1,X))g(\hat{\theta}_1,X)] 
   \end{split}
    \label{proof1:eq1}
\end{equation}

 Since $\hat{\theta}_0$ and $\hat{\theta}_1$ are identical draws from $\mathbb{P}_{\hat{\theta}}$  we obtain
 \begin{equation}\mathbb{E}_{X,\hat{\theta}_{1}}[g(\hat{\theta}_1,X)] =\mathbb{E}_{X,\hat{\theta}_{0}}[g(\hat{\theta}_0,X)] =  \mathbb{E}_{X,\hat{\theta}}[g(\hat{\theta},X)] = \mathbb{E}_{\hat{\theta},X}[\hat{Y}] 
 \label{proof1:eq2}
 \end{equation}
 
 Therefore, we can use the above condition \eqref{proof1:eq2} to obtain
 
\begin{equation}\mathbb{E}_{X,\hat{\theta}}[\hat{Y}] = \mathbb{E}_{X,\hat{\theta}_0,\hat{\theta}_1}[\frac{1}{2}g(\hat{\theta}_0,X) + \frac{1}{2}g(\hat{\theta}_1,X)]  
\label{proof1:eq3}
\end{equation}

Take the difference of LHS of \eqref{proof1:eq1} and \eqref{proof1:eq3} to obtain 

\begin{equation}
       \mathbb{E}_{X,\tilde{\theta}}[\tilde{Y}] - \mathbb{E}_{X,\hat{\theta}}[\hat{Y}] = \mathbb{E}_{X,\hat{\theta}_0, \hat{\theta}_1}[(p(\hat{\theta}_0, \hat{\theta}_1,X)-\frac{1}{2})(g(\hat{\theta}_0,X)   -g(\hat{\theta}_1,X))] 
\end{equation}
 
  Therefore,  $\mathbb{E}_{X,\hat{\theta}_0, \hat{\theta}_1}[(p(\hat{\theta}_0, \hat{\theta}_1,X)-\frac{1}{2})(g(\hat{\theta}_0,X)   -g(\hat{\theta}_1,X))] <0 \iff  \mathbb{E}_{X,\tilde{\theta}}[\tilde{Y}] - \mathbb{E}_{X,\hat{\theta}}[\hat{Y}] <0 $
 This proves the theorem. 
 \end{proof}

 We restate Proposition \ref{prop2} below.
 \begin{proposition}
 Assumption \ref{assm2}  $\implies \mathbb{E}_{X,\hat{\theta}_0, \hat{\theta}_1}[\min\{\hat{Y}_0, \hat{Y}_1\}] \leq  \mathbb{E}_{X,\tilde{\theta}}[\tilde{Y}]  \leq \gamma \mathbb{E}_{X,\hat{\theta}_0, \hat{\theta}_1}[\min\{\hat{Y}_0, \hat{Y}_1\}] + (1-\gamma)\mathbb{E}_{X,\hat{\theta}_0, \hat{\theta}_1}[ \max\{\hat{Y}_{0}, \hat{Y}_1\}]$
 \label{prop2_append}
\end{proposition}

% \textbf{Proof}
\begin{proof}
From \eqref{proof1:eq1} it follows that 
\begin{equation}
\begin{split}
  \mathbb{E}_{X,\tilde{\theta}}[\tilde{Y}]  = \mathbb{E}_{X,\hat{\theta}_0, \hat{\theta}_1}[p(\hat{\theta}_0, \hat{\theta}_1,X)\hat{Y}_0   + (1-p(\hat{\theta}_0, \hat{\theta}_1,X))\hat{Y}_1] 
  \end{split}
\end{equation}

 Consider the term inside the expectation given above.
 
 Suppose $\hat{Y}_{0} \leq \hat{Y}_{1}$
then from Assumption \ref{assm2} it follows
 
 \begin{equation}
     p(\hat{\theta}_0, \hat{\theta}_1,X)\hat{Y}_0   + (1-p(\hat{\theta}_0, \hat{\theta}_1,X))\hat{Y}_1 \leq \gamma \hat{Y}_0 + (1-\gamma)\hat{Y}_1 
     \label{proof2:eqn1}
 \end{equation}
 Suppose $\hat{Y}_1<\hat{Y}_0$ then from Assumption \ref{assm2} it follows
  \begin{equation}
     p(\hat{\theta}_0, \hat{\theta}_1,X)\hat{Y}_0   + (1-p(\hat{\theta}_0, \hat{\theta}_1,X))\hat{Y}_1 \leq \gamma \hat{Y}_1 + (1-\gamma)\hat{Y}_0 
     \label{proof2:eqn2}
 \end{equation}
 
 Combining equations \eqref{proof2:eqn1} and \eqref{proof2:eqn2}, we get 

 \begin{equation}
 \begin{split}
   &   \mathbb{E}_{X, \hat{\theta}_{0}, \hat{\theta}_1}[  p(\hat{\theta}_0, \hat{\theta}_1,X)\hat{Y}_0   + (1-p(\hat{\theta}_0, \hat{\theta}_1,X))\hat{Y}_1] \leq \\& \gamma \mathbb{E}_{X, \hat{\theta}_{0}, \hat{\theta}_1}[\min\{\hat{Y}_0, \hat{Y}_1\}] + (1-\gamma)\mathbb{E}_{X, \hat{\theta}_{0}, \hat{\theta}_1}[\max \{\hat{Y}_0,\hat{Y}_1\}]
   \end{split}
 \end{equation}

LHS of the above Proposition follows trivially. Observe that $\tilde{Y}$ takes one of the values $\hat{Y}_0$ or $\hat{Y}_1$. Therefore, $\min\{\hat{Y}_0, \hat{Y}_1\} \leq \tilde{Y} \implies \mathbb{E}_{X, \hat{\theta}_0, \hat{\theta}_1}[\min\{\hat{Y}_0, \hat{Y}_1\}] \leq \mathbb{E}_{X, \hat{\theta}_0, \hat{\theta}_1}[\tilde{Y}] $. 
 \end{proof}
%  $\hfill$ $\blacksquare$
 
 We restate Proposition \ref{prop3_van} below.
 \begin{proposition}

If $\rho(p_{\mathsf{vnila}}(\hat{\theta}_0, \hat{\theta}_1,X),\hat{Y}_1-\hat{Y}_0)>0$, then the Vanilla approach performs better than the independent random initializer, i.e.,  $\mathbb{E}_{X, \tilde{\theta}_{\mathsf{vnila}}}[\tilde{Y}_{\mathsf{vnila}} ]< \mathbb{E}_{X, \hat{\theta}}[\hat{Y} ]$.

\end{proposition}

\begin{proof}

Since the Pearson correlation is positive $  \mathbb{E}_{X, \hat{\theta}_{0}, \hat{\theta}_{1}}\Big[(p(\hat{\theta}_0, \hat{\theta}_1,X)) (\hat{Y}_1 - \hat{Y}_0) \Big] >0 $, $\implies$
$  \mathbb{E}_{X, \hat{\theta}_{0}, \hat{\theta}_{1}}\Big[(p(\hat{\theta}_0, \hat{\theta}_1,X)-\frac{1}{2}) (\hat{Y}_0 - \hat{Y}_1) \Big] <0 $
where we use the fact $\mathbb{E}[\hat{Y}_0 - \hat{Y}_1]=0$. The rest follows from Proposition \ref{prop1}.

\end{proof}
 We restate Proposition \ref{prop4} below.
\begin{proposition}
If   $\rho(Z_{k}, \hat{Y}_{M-1}-\hat{Y}_{k})>0$  $\forall k\in\{0,\cdots, M-2\}$, then Val-Init performs better  than the independent random initializer, i.e., $\mathbb{E}_{X, \tilde{\theta}_{\mathsf{val}}}[\tilde{Y}_{\mathsf{val}} ]< \mathbb{E}_{X, \hat{\theta}}[\hat{Y} ]$. 
\end{proposition}

\begin{proof} Simplify $g(\tilde{\theta}_{\mathsf{val}},X)$ to obtain 

\begin{equation}
    \begin{split}
      \tilde{Y}_{\mathsf{val}} =g(\tilde{\theta}_{\mathsf{val}},X) &= g(\sum_{k=0}^{M-1}Z_k\hat{\theta}_k,X) =\sum_{k=0}^{M-1}Z_kg(\hat{\theta}_k,X)  = \sum_{k=0}^{M-1}Z_k\hat{Y}_k
    \end{split}
\end{equation}

Define a discrete uniform random variable $W$ that takes one of the  values in the set $\{0,.., M-1\}$. Define $R_{k}$ as a random variable which is indicates if $W=k$, i.e., if $W=k$, then $R_{k}=1$, else it is zero. Observe that $\sum_{k=0}^{M-1}R_{k}=1$. 
Define the objective value achieved by random initializer as $\hat{Y} = \sum_{k=0}^{M-1}R_k \hat{Y}_k$.
Take the difference between $ \tilde{Y}_{\mathsf{val}}$ and $\hat{Y}^{'}$ to obtain

\begin{equation}
    \begin{split}
        \tilde{Y}_{\mathsf{val}}- \hat{Y}^{'} &= \sum_{k=0}^{M-1}(Z_k-R_k)\hat{Y}_k \\
    &= \sum_{k=0}^{M-2}(Z_k-R_k)\hat{Y}_k  + (Z_{M-1}-R_{M-1})\hat{Y}_{M-1} \\
    &= \sum_{k=0}^{M-2}(Z_k-R_k)(\hat{Y}_k  - \hat{Y}_{M-1})
    \end{split}
    \label{prop3:eqn1}
\end{equation}
In the above simplfication, we used $R_{M-1} = 1-\sum_{k=0}^{M-2}R_{k}$ and $Z_{M-1} = 1-\sum_{k=0}^{M-2}Z_{k}$. Recall that the correlation between two random variables $X$ and $Y$ is given as $\rho(X,Y) = \frac{E[XY]-E[X][Y]}{\sigma_X \sigma_Y}$, where $\sigma_X$ and $\sigma_Y$ is the standard deviation of $X$ and $Y$. We use this identity in our simplification next.

Let us consider the expectation of one of the terms inside the summation above \eqref{prop3:eqn1} 

\begin{equation}
    \begin{split}
        & \mathbb{E}[(Z_k-R_k)(\hat{Y}_k  - \hat{Y}_{M-1})] = \\
        & \mathbb{E}[(Z_k-R_k)]\mathbb{E}[(\hat{Y}_k  - \hat{Y}_{M-1})] + \rho(Z_k-R_k,\hat{Y}_k  - \hat{Y}_{M-1})\sigma_{Z_{k}-R_{k}}\sigma_{\hat{Y}_k  - \hat{Y}_{M-1}} = \\ 
        &  \rho(Z_k-R_k,\hat{Y}_k  - \hat{Y}_{M-1})\sigma_{Z_{k}-R_{k}}\sigma_{\hat{Y}_k  - \hat{Y}_{M-1}}  <0
    \end{split}
    \label{proof4:eqn1}
\end{equation}
In the above simplfication \eqref{proof4:eqn1}, we use the following  facts i) $\mathbb{E}[(\hat{Y}_k  - \hat{Y}_{M-1})] =0$ because $\hat{Y}_k $ and $ \hat{Y}_{M-1}$ are identically distributed, ii) the correlation $\rho(Z_k-R_k,\hat{Y}_k  - \hat{Y}_{M-1})<0$, and iii) standard deviations $\sigma_{Z_{k}-R_{k}}$, $\sigma_{\hat{Y}_k  - \hat{Y}_{M-1}}$ are positive. 
\end{proof}

\begin{lemma}

\label{lemma1}
If Assumption \ref{assm4} holds, then  
    \begin{equation}
    \mathbb{P}_{\theta}\Big(|y- h_{\mathsf{val}}(\hat{\theta},x)|< \Delta/2 \Big|  \hat{\theta} \in \Phi_{x}(y), X=x \Big) > 1-\zeta
    \label{eqn:val_bnd}
\end{equation}
\end{lemma}
\begin{proof}
We rewrite the probability in the lemma as follows
\begin{equation}
\begin{split}
     &     \mathbb{P}_{\hat{\theta}}\Big(|y- h_{\mathsf{val}}(\hat{\theta},x)|< \Delta/2 \Big|  \hat{\theta} \in \Phi_{x}(y), X=x \Big)   \\
     = &     \mathbb{P}_{\hat{\theta}}\Big(|y- h_{\mathsf{val}}(\hat{\theta},x)|^2< \Delta^2/4 \Big|  \hat{\theta} \in \Phi_{x}(y), X=x \Big) \\
     = & 1-  \mathbb{P}_{\hat{\theta}}\Big(|y- h_{\mathsf{val}}(\hat{\theta},x)|^2\geq  \Delta^2/4 \Big|  \hat{\theta} \in \Phi_{x}(y), X=x \Big)
\end{split}
\label{eqn1:lemma1}
\end{equation}
We simplify the second term in the last expression above \eqref{eqn1:lemma1} using Markov's inequality 
\begin{equation}
\begin{split}
     & \mathbb{P}_{\hat{\theta}}\Big(|y- h_{\mathsf{val}}(\hat{\theta},x)|^2\geq \Delta^2/4 \Big|  \hat{\theta} \in \Phi_{x}(y), X=x \Big) \leq \\ &\frac{\mathbb{E}[|y- h_{\mathsf{val}}(\hat{\theta},x)|^2\Big|  \hat{\theta} \in \Phi_{x}(y), X=x]}{\Delta^2/4}
\end{split}
\label{eqn2:lemma1}
\end{equation}
We use Assumption \ref{assm4} to simplify \eqref{eqn2:lemma1} as follows
\begin{equation}
\begin{split}
  &  \frac{\mathbb{E}[|y- h_{\mathsf{val}}(\hat{\theta},x)|^2 \Big|  \hat{\theta} \in \Phi_{x}(y), X=x]}{\Delta^2/4} \leq   \frac{\Delta^{2}/4(\zeta)}{\Delta^2/4} \leq \zeta
    \end{split}
    \label{eqn3:lemma1}
\end{equation}
From \eqref{eqn2:lemma1} and \eqref{eqn3:lemma1}, it follows that 
\begin{equation}
    \mathbb{P}_{\theta}\Big(|y- h_{\mathsf{val}}(\hat{\theta},x)|^2\geq \Delta^2/4 \Big|  \hat{\theta} \in \Phi_{x}(y), X=x\Big) \leq \zeta< \frac{\tilde{\epsilon}}{Mf^{\mathsf{sup}}}
\end{equation}
From \eqref{eqn1:lemma1} it follows that
\begin{equation}
\mathbb{P}_{\theta}\Big(|y- h_{\mathsf{val}}(\hat{\theta},x)|< \Delta/2 \Big|  \hat{\theta} \in \Phi_{x}(y), X=x \Big)  > 1-\frac{\tilde{\epsilon}}{f^{\mathsf{sup}}M} 
\end{equation}
This completes the proof.
\end{proof}
We restate Proposition \ref{prop5} below.
\begin{proposition}For problem instance $x$, if Assumptions \ref{assm3} and \ref{assm4} hold, then Val-Init is an $\tilde{\epsilon}$-approximation of the multi-start based approach:
\begin{equation*}
\mathbb{E}_{ \{\hat{\theta}_k\}_{k=0}^{M-1}}[\tilde{Y}_{\mathsf{val}}|X=x] \leq  \mathbb{E}_{\{\hat{\theta}_k\}_{k=0}^{M-1}}[\min_{k \in \{0,..,M-1\}}\{\hat{Y}_{k}\} |X=x]+ \tilde{\epsilon}.
\end{equation*}
% where $\epsilon$ is  the approximation error.
\end{proposition}

\begin{proof} We will consider the case when $M=2$. The same steps of the proof can be repeated for the general case when $M>2$. 
We can write $\tilde{Y}_{\mathsf{val}} = (1-Z)\hat{Y}_{0} + Z\hat{Y}_1$, where $Z$ is the selection variable decided by Val-Init algorithm.
\begin{equation}
\begin{split}
    & \mathbb{E}_{\hat{\theta}_0, \hat{\theta}_1}[\tilde{Y}_{\mathsf{val}}|X=x] = \sum_{y_0\in \mathcal{F}(x), y_1\in \mathcal{F}(x)} \mathbb{P}[\hat{Y}_{0}=y_0, \hat{Y}_1=y_1]\mathbb{E}_{\hat{\theta}_0, \hat{\theta}_1}[\tilde{Y}_{\mathsf{val}}|X=x, \hat{Y}_{0}=y_0, \hat{Y}_1=y_1] \\ 
    &= \sum_{y_0\in \mathcal{F}(x), y_1\in \mathcal{F}(x)} \mathbb{P}[\hat{Y}_{0}=y_0, \hat{Y}_1=y_1]\mathbb{E}_{\hat{\theta}_0, \hat{\theta}_1}[\tilde{Y}_{\mathsf{val}}|X=x, \hat{\theta}_0\in \Phi_{x}(y_0), \hat{\theta}_1\in \Phi_{x}(y_1)] \\
    & = \sum_{y_0\in \mathcal{F}(x), y_1\in \mathcal{F}(x)} \mathbb{P}[\hat{Y}_{0}=y_0, \hat{Y}_1=y_1]\mathbb{E}_{\hat{\theta}_0, \hat{\theta}_1}[\tilde{Y}_{\mathsf{val}}|X=x, \hat{\theta}_0\in \Phi_{x}(y_0), \hat{\theta}_1\in \Phi_{x}(y_1)]
\end{split}
\label{proof6:eqn1}
\end{equation}
\begin{equation}
\begin{split}
  &  \mathbb{E}_{\hat{\theta}_0, \hat{\theta}_1}[\tilde{Y}_{\mathsf{val}}|X=x, \hat{\theta}_0\in \Phi_{x}(y_0), \hat{\theta}_1\in \Phi_{x}(y_1)] \\ 
  = &   \mathbb{E}_{\hat{\theta}_0, \hat{\theta}_1}[(1-Z)y_0 + Zy_1|X=x, \hat{\theta}_0\in \Phi_{x}(y_0), \hat{\theta}_1\in \Phi_{x}(y_1)] \\
 = & \mathbb{P}[h_{\mathsf{val}}(\hat{\theta}_0,x)<h_{\mathsf{val}}(\hat{\theta}_1,x)|X=x, \hat{\theta}_0\in \Phi_{x}(y_0), \hat{\theta}_1\in \Phi_{x}(y_1)]y_0 +  \\ 
 &\mathbb{P}[h_{\mathsf{val}}(\hat{\theta}_0,x)>h_{\mathsf{val}}(\hat{\theta}_1,x)|X=x, \hat{\theta}_0\in \Phi_{x}(y_0), \hat{\theta}_1\in \Phi_{x}(y_1)]y_1
%   = &  \mathbb{P}(h_{\mathsf{val}}(\hat{\theta}_0,x)<h_{\mathsf{val}}(\hat{\theta}_1,x))y_0 + \mathbb{P}(h_{\mathsf{val}}(\hat{\theta}_0,x)>h_{\mathsf{val}}(\hat{\theta}_1,x)|X=x, \hat{\theta}_0\in \Phi_{x}(y_0), \hat{\theta}_1\in \Phi_{x}(y_1))y_1|X=x, \hat{\theta}_0\in \Phi_{x}(y_0), \hat{\theta}_1\in \Phi_{x}(y_1)]
\end{split}
\end{equation}

Without loss of generality say $y_0<y_1$ 
\begin{equation}
\begin{split}
   & \mathbb{P}\Big[h_{\mathsf{val}}(\hat{\theta}_0,x)<h_{\mathsf{val}}(\hat{\theta}_1,x)|X=x, \hat{\theta}_0\in \Phi_{x}(y_0), \hat{\theta}_1\in \Phi_{x}(y_1)\Big] >  \\
    &  \mathbb{P}\Big[|h_{\mathsf{val}}(\hat{\theta}_0,x)-y_0| < \frac{\Delta}{2}, |h_{\mathsf{val}}(\hat{\theta}_1,x)-y_1| < \frac{\Delta}{2} \Big|X=x, \hat{\theta}_0\in \Phi_{x}(y_0), \hat{\theta}_1\in \Phi_{x}(y_1)\Big]
\end{split}
\end{equation}
$h_{\mathsf{val}}(\hat{\theta}_0,x)$ and $h_{\mathsf{val}}(\hat{\theta}_1,x)$ are both conditionally independent given $\hat{\theta}_0\in \Phi_{x}(y_0), \hat{\theta}_1\in \Phi_{x}(y_1)$ (can be observed by decomposing the joint distributions). As a result, we can simplify the above expression as follows. 

\begin{equation}
\begin{split}
    &  \mathbb{P}\Big[|h_{\mathsf{val}}(\hat{\theta}_0,x)-y_0| < \frac{\Delta}{2}, |h_{\mathsf{val}}(\hat{\theta}_1,x)-y_1| < \frac{\Delta}{2} \Big|X=x, \hat{\theta}_0\in \Phi_{x}(y_0), \hat{\theta}_1\in \Phi_{x}(y_1)\Big] \\ 
    &  \mathbb{P}\Big[|h_{\mathsf{val}}(\hat{\theta}_0,x)-y_0| < \frac{\Delta}{2}  \Big|X=x, \hat{\theta}_0\in \Phi_{x}(y_0)\Big]
    \mathbb{P}\Big[|h_{\mathsf{val}}(\hat{\theta}_1,x)-y_1| < \frac{\Delta}{2}  \Big|X=x, \hat{\theta}_1\in \Phi_{x}(y_1)\Big] \\
    & = (1-\frac{\epsilon}{2f^{\mathsf{sup}}})^2 >1-\frac{\epsilon}{f^{\mathsf{sup}}}
\end{split}
\end{equation}

Therefore, when $y_0 <y_1$ we get 
\begin{equation}
    \begin{split}
        \mathbb{E}_{\hat{\theta}_0, \hat{\theta}_1}[\tilde{Y}_{\mathsf{val}}|X=x, \hat{\theta}_0\in \Phi_{x}(y_0), \hat{\theta}_1\in \Phi_{x}(y_1)] \leq (1-\frac{\epsilon}{f^{\mathsf{sup}}})y_0 + \frac{\epsilon}{f^{\mathsf{sup}}}y_{1}
    \end{split}
\end{equation}
In general we can write 
\begin{equation}
    \begin{split}
        \mathbb{E}_{\hat{\theta}_0, \hat{\theta}_1}[\tilde{Y}_{\mathsf{val}}|X=x, \hat{\theta}_0\in \Phi_{x}(y_0), \hat{\theta}_1\in \Phi_{x}(y_1)] &\leq  (1-\frac{\epsilon}{f^{\mathsf{sup}}})\min\{y_0,y_1\} + \frac{\epsilon}{f^{\mathsf{sup}}}\max\{y_0,y_1\} \\ 
        &\leq  \min\{y_0,y_1\} + \epsilon
    \end{split}
    \label{proof6:eqn2}
\end{equation}

Substituting \eqref{proof6:eqn2} in \eqref{proof6:eqn1} to obtain
\begin{equation}
\mathbb{E}_{ \{\hat{\theta}_k\}_{k=0}^{M-1}}[\tilde{Y}_{\mathsf{val}}|X=x] \leq  \mathbb{E}_{ \{\hat{\theta}_k\}_{k=0}^{M-1}}[\min_{k \in \{1,..,M\}}\{\hat{Y}_{k}\} |X=x]+ \epsilon
\end{equation}

\end{proof}
We restate Proposition \ref{prop6} below.
\begin{proposition} If Assumption \ref{assm5} holds and the prediction error of $h_{\mathsf{arg}}$ is small, i.e.  $0<s<1$, $0<\eta<1$, 
$\mathbb{E}_{ X,\hat{\theta}}\big[\|h_{\mathsf{arg}}(\hat{\theta},X) - g^{\dagger}(\hat{\theta}, X)\|^2\big] \leq s (1-\epsilon)(\delta \eta)^2$, where $\epsilon$, $\delta$ are from Assumption \ref{assm5}, then with probability $(1-s)(1-\epsilon)$  Arg-Init  achieves $\eta$-factor reduction.
\end{proposition}

 \begin{proof}
Consider the event $$\big\|h_{\mathsf{arg}}(\hat{\theta},X) - g^{\dagger}(\hat{\theta},X)\big\|^2< (\delta\eta)^2$$ and $\hat{\theta} \not\in B_{\delta}$, i.e. 
$$\big\|\hat{\theta} - g^{\dagger}(\hat{\theta},X)\big\|^2\geq (\delta)^2$$

The above event is equivalent to saying that the initializer from $\hat{\theta}$ starts at a larger distance (a factor $\eta$ larger) from the true solution than the one found by the learned model. 

We use Markov's inequality on the complement of the first event above as follows
\begin{equation}
\begin{split}
&\mathbb{P}_{\hat{\theta},X}\Bigg[ \big\|h_{\mathsf{arg}}(\hat{\theta},X) - g^{\dagger}(\hat{\theta},X)\big\|^2\geq (\delta\eta)^2\; \Big|\; \hat{\theta} \not \in B_{\delta}\Bigg] \leq  \frac{\mathbb{E}_{\hat{\theta}, X}\Bigg[\|h_{\mathsf{arg}}(\hat{\theta},X) - g^{\dagger}(\hat{\theta},X)\|^2 \;\Big|\; \hat{\theta} \not \in B_{\delta}\Bigg]}{(\delta\eta)^2} 
\end{split}
\label{proof7:eqn1}
\end{equation}

Next we use Assumption \ref{assm5} to get a bound on $$\mathbb{E}_{\hat{\theta},X}\Bigg[\Big\|h_{\mathsf{arg}}(\hat{\theta},X) - g^{\dagger}(\hat{\theta},X)\Big\|^2 \;\Big|\; \hat{\theta} \not \in B_{\delta}\Bigg]$$
as follows
 \begin{equation}
     \begin{split}
      &  \mathbb{E}_{\hat{\theta}, X}\Bigg[\Big\|h_{\mathsf{arg}}(\hat{\theta},X) - g^{\dagger}(\hat{\theta},X)\Big\|^2\Bigg]  \\
      &=   \mathbb{E}_{\hat{\theta}, X}\Bigg[\Big\|h_{\mathsf{arg}}(\hat{\theta},X) -g^{\dagger}(\hat{\theta},X)\|^2\; \Big| \;\hat{\theta} \not \in B_{\delta}\Bigg]  \times\mathbb{P}[\hat{\theta} \not \in B_{\delta}]+ \\  
      & \Big(1-\mathbb{P}[\hat{\theta} \not \in B_{\delta}]\Big) \mathbb{E}_{\hat{\theta},X}\Bigg[\Big\|h_{\mathsf{arg}}(\hat{\theta},X) -g^{\dagger}(\hat{\theta},X)\|^2\; \Big| \;\hat{\theta} \in B_{\delta}\Bigg] \\
          \implies & \mathbb{E}_{\hat{\theta},X}\Big[\|h_{\mathsf{arg}}(\hat{\theta},X) - g^{\dagger}(\hat{\theta},X)\|^2 \;\Big| \;\hat{\theta} \not \in B_{\delta}\Big]    \leq 1/(1-\epsilon)\mathbb{E}_{\hat{\theta},X}\Big[\|h_{\mathsf{arg}}(\hat{\theta},X) - g^{\dagger}(\hat{\theta},X)\|^2\Big]
     \end{split}
     \label{proof7:eqn2}
 \end{equation}
 Using the above equation \eqref{proof7:eqn1} in the RHS in Markov inequality \eqref{proof7:eqn2}

 \begin{equation}
\begin{split}
&\mathbb{P}\Big[\Big\|h_{\mathsf{arg}}(\hat{\theta},X) - g^{\dagger}(\hat{\theta},X)\Big\|^2\geq (\delta\eta)^2 \;\Big|\; \hat{\theta} \not \in B_{\delta}\Big]   \leq \frac{\mathbb{E}_{\hat{\theta},X}\Big[\|h_{\mathsf{arg}}(\hat{\theta},X) - g^{\dagger}(\hat{\theta},X)\|^2\Big]}{(1-\epsilon)(\delta\eta)^2} 
\end{split}
\end{equation}
Substitute the condition on the prediction error stated in the statement of the Proposition to get 
 \begin{equation}
\begin{split}
&\mathbb{P}\Bigg[\Big\|h_{\mathsf{arg}}(\hat{\theta},X) - g^{\dagger}(\hat{\theta},X)\Big\|^2\geq (\delta\eta)^2 \;|\; \hat{\theta} \not \in B_{\delta}\Bigg]\leq s 
\end{split}
\label{thm3:eqn24}
\end{equation}

We simplify the probability of the event defined (the event leads to $\eta$-factor reduction) in the beginning of the proof  as
\begin{equation}
\begin{split}
     \mathbb{P}\Big[&\big\|h_{\mathsf{arg}}(\hat{\theta},X) - g^{\dagger}(\hat{\theta},X)\big\|^2\leq (\delta\eta)^2 ,  \hat{\theta} \not\in B_{\delta}\Big]  \\
  =\; &  \mathbb{P}\Big[\big\|h_{\mathsf{arg}}(\hat{\theta},X) - g^{\dagger}(\hat{\theta},X)\big\|^2\leq (\delta\eta)^2 \;\big|\; \hat{\theta} \not \in B_{\delta}\Big]    \times \mathbb{P}\Big[\hat{\theta} \not \in B_{\delta} \Big]\\
  \geq \;&  \mathbb{P}\Big[\big\|h_{\mathsf{arg}}(\hat{\theta},X) - g^{\dagger}(\hat{\theta},X)\big\|^2\leq (\delta\eta)^2 \;\big| \;\hat{\theta} \not \in B_{\delta}\Big]   (1-\epsilon)
\end{split}
\label{thm3:eqn25}
\end{equation}
We subtitute equation \eqref{thm3:eqn24} in \eqref{thm3:eqn25} equation to obtain
\begin{equation}
\begin{split}
&\mathbb{P}\Big[\big\|h_{\mathsf{arg}}(\hat{\theta},X) - g^{\dagger}(\hat{\theta},X)
\big\|^2\leq (\delta\eta)^2\; \bigcap \; \hat{\theta} \not\in B_{\delta}\Big] \geq  (1-s)(1-\epsilon)
\end{split}
\end{equation}
This completes the proof.

\end{proof}

\subsection{Connection between Proposition \ref{prop3_van} and uniform random binary classifier}

In Proposition \ref{prop3_van}, we show that correlation between model's probabilities and the label is positive, then the model is better at initialization than a random classifier. In standard binary classification problems as well, a similar result exists, i.e., if the classifier's output is positively correlated with the label, then it performs better than uniform random classification. We derive his result next for completeness.

We are given labeled data with labels $Y \in \{1,-1\}$. Suppose the two classes occur with equal probability then $
\mathbb{E}[Y]= 0$. Define the output of the classifier as $P \in \{1,-1\}$. 

Define $A = \frac{1+YP}{2}$. If $Y=1,P=1$ or $Y=-1, P=-1$, then $A=1$. If $Y=1, P=-1$ or $Y=-1, P=-1$, then $A=0$. Observe that $\mathbb{E}[A]$ is the accuracy of the classifier $P$. 
Let us simplify $\mathbb{E}[YP]$. $\mathbb{E}[YP] = \mathbb{E}[Y] \mathbb{E}[P] +   \rho(Y,P) \sigma_Y \sigma_P$, where $\rho(Y,P)$ is the correlation between $Y$ and $P$, $\sigma_Y$ and $\sigma_P$ are the standard deviation of $Y$ and $P$.

Simplify $$\mathbb{E}[A] = \frac{1 + \mathbb{E}[YP]}{2} = \frac{1 + \rho(Y,P)\sigma_Y\sigma_P}{2}$$
If the correlation is positive $\rho(Y,P)>0$, the accuracy is greater than $\frac{1}{2}$. A uniform random binary classifier has an accuracy of $\frac{1}{2}$. If correlation  is negative $\rho(Y,P)<0$, then accuracy is less than $\frac{1}{2}$ but can be made greater than $\frac{1}{2}$ by flipping the sign of the output of $P$.

\subsection{Correlation between $m$ and $\hat{Y}_1-\hat{Y}_0$}
Below Proposition \ref{prop3_van}, we stated that it is reasonable to expect a positive correlation between $m$ and $\hat{Y}_1-\hat{Y}_0$.
$m$ is trained to predict $\hat{Y}_1-\hat{Y}_{0}$. For simplicity let us call $\hat{Y}_{\mathsf{diff}} = \hat{Y}_1-\hat{Y_0}$. 
Suppose the loss function used to train $m$ is mean squared error $\mathbb{E}\big[(\hat{Y}_{\mathsf{diff}}- m(X,\hat{\theta}_0, \hat{\theta}_1))^2\big]$.
$$\mathbb{E}\Big[(\hat{Y}_{\mathsf{diff}}- m(X,\hat{\theta}_0, \hat{\theta}_1))^2\Big] = \mathbb{E}\Big[(\hat{Y}_{\mathsf{diff}})^2\Big] + \mathbb{E}\Big[(m(X,\hat{\theta}_0, \hat{\theta}_1))^2\Big]-2\mathbb{E}\Big[\hat{Y}_{\mathsf{diff}}(m(X,\hat{\theta}_0, \hat{\theta}_1))\Big]$$
$$ = \mathbb{E}\Big[(\hat{Y}_{\mathsf{diff}})^2\Big] + \mathbb{E}\Big[(m(X,\hat{\theta}_0, \hat{\theta}_1))^2\Big]-2\rho(Y_{\mathsf{diff}},m) \sigma_{Y_{\mathsf{diff}}}\sigma_{m(X,\hat{\theta}_0, \hat{\theta}_0)}$$

Consider we are given a function $m^{'}(X,\hat{\theta}_0, \hat{\theta}_1)$. Our task is to learn a function $m$, which is a scalar multiple of $m^{'}$ defined as $m(X,\hat{\theta}_0, \hat{\theta}_1)) = a m^{'}(X,\hat{\theta}_0, \hat{\theta}_1)$, such that it minimizes the loss defined above.

Suppose $\rho(Y_{\mathsf{diff}},m^{'})$ is the correlation between $m^{'}$ and $\hat{Y}_{\mathsf{diff}}$.
If $\rho(Y_{\mathsf{diff}},m^{'})=0$, then the optimal $a=0$. If $\rho(Y_{\mathsf{diff}},m^{'})>0$, then  the optimal $a>0$ and thus $\rho(Y_{\mathsf{diff}},m)>0$. If  $\rho(Y_{\mathsf{diff}},m^{'})<0$, then the optimal $a<0$ and thus $\rho(Y_{\mathsf{diff}},m)>0$. Therefore, if $\rho(Y_{\mathsf{diff}},m^{'})\not=0$, then the $\rho(Y_{\mathsf{diff}},m)>0$ and if $\rho(Y_{\mathsf{diff}},m^{'})=0$, then optimal model $m=0$. As long as any learning happens, i.e. a non-zero model $m$ is learned it has a finite and positive correlation with $Y_{\mathsf{diff}}$. The condition in Corollary 1 requires $\frac{e^{m}}{1+e^{m}}$ to be positively correlated with $Y_{\mathsf{diff}}$.  \cite{egozcue2009some} establishes sufficient conditions ($Y_{\mathsf{diff}}$ and $m(X,\hat{\theta}_0, \hat{\theta}_1))$ have  positive quadrant dependency) under which  even under simple transformations of a variable the sign of the correlation is preserved.

\section{Supplement to Experiments}
The experiments were conducted on 2.3 GHz Intel Core i9 with 32 GB RAM (2400 MHz and DDR4). We used keras to train all the models. For any hyperparameter that we do not specify, we use the default values (for e.g., initializer for NN is set to Glorot uniform). 
We organize this section as follows. There are five problem families that we work with -- a) convex optimization (minmize linear form of the problem in equation \eqref{eqn: stat_opt_example1}, b) Ackley function minimization, c) generating adversarial examples, d) generating contrastive explanations, e) sum-rate optimization. We start by describing the datasets and the models for which we generate adversarial examples and contrastive explanations. After that we describe the different initializers that were used, which is followed by the description of the gradient descent solvers for each of the problems. Lastly, we provide the supplementary experiments results.  

\subsection{Datasets and Models for Adversarial Examples and Contrastive Explanations}
\label{sup_mode}
% In this section, we provide some of the details that we did not provide in the main manuscript due to the space limitations. 

\textbf{CIFAR-10 dataset and model.} The data can be downloaded from \footnote{\url{https://www.cs.toronto.edu/~kriz/cifar.html}}. We use the default train-test split. We train a multi-layer convolutional neural network. The first layer uses 32 convolutional filters, with kernel size $(3,3)$ and ReLU activation. It is followed by a maxpool layer.  The third layer uses 64 convolutional filters, with kernel size $(3,3)$ and ReLU activation. It is followed by  a maxpool layer. The last layer flattens the output from third layer, outputs softmax scores for the $10$ classes onto the $10$ output nodes. The model is then trained using cross-entropy loss function with Adam optimizer. We use a learning rate of 1e-3. The total number of training epochs were set to 10. The trained  model achieves 70 percent test accuracy.

\textbf{MNIST digits dataset} The data can be downloaded from \footnote{\url{https://www.tensorflow.org/api_docs/python/tf/keras/datasets/mnist}}. We use the standard split in MNIST digits. 
We use a two layer neural network, where the first layer (the hidden layer) has 128 nodes with ReLU activation, and the second layer (the output layer) has 10 output nodes with softmax activation. The model is trained using cross-entropy loss function with Adam optimizer. We set a learning rate of 1e-3 and the other parameters of Adam are set to default value. The total number of training epochs were set to 10.
The trained  model achieves 97 percent accuracy.

\textbf{HELOC dataset and model.} The dataset can be downloaded from \footnote{\url{https://community.fico.com/s/explainable-machine-learning-challenge?tabset-3158a=2}}. Home equity line of credit (HELOC) data contains the information about the applicant in their credit report and their loan repayment history. We  build a model to predict whether applicants will make timely payments.  We divided the data into 75 percent train, 25 percent test. We use a two layer neural network, where the first layer (the hidden layer) has 200 nodes with ReLU activation and the second layer (the output layer) has two output nodes with softmax activation. The model is trained using cross-entropy loss function with Adam optimizer. We set a learning rate of 1e-3 and the other parameters of Adam are set to default value. The total number of training epochs were set to 20.  The trained   model achieves 72 percent accuracy.

\textbf{Waveform dataset} The dataset can be downloaded from \footnote{\url{https://archive.ics.uci.edu/ml/machine-learning-databases/waveform/}}. The dataset consists of 40 attributes that help identify the wave type (from three types of waves). We divided the data into 75 percent train, 25 percent test. We use a two layer neural network, where the first layer (the hidden layer) has 200 nodes with ReLU activation and the second layer (the output layer) has three output nodes with softmax activation. The model is trained using cross-entropy loss function with Adam optimizer. We set a learning rate of 1e-3 and the other parameters of Adam are set to default value. The total number of training epochs were set to 20.  The trained  model achieves 85 percent accuracy.

\subsection{Initializers}
\label{init_sec}
We discuss the different initializers. Each of these initializers pick a value for $\theta_{\mathsf{in}}$ for gradient descent solver in the Algorithm 1.  

\textbf{Zero initializer.} This is the simplest of the benchmarks. We start gradient descent with zero vector as initialization. Note that when using it for adversarial examples, and the contrastive explanations, we use the loss function in \cite{cheng2018query} whose gradient involves $\frac{\theta}{\|\theta\|}$, where $\theta$ is the perturbation. This term undefined at zero, which is why we add a very small noise term to zero vector to avoid undefined gradient values. 

\textbf{Random initializer.} In this benchmark, we have to choose a distribution from which we sample the random initializations.  For convex optimization, we initialize the optimization variable $\theta$ in equation \eqref{eqn: stat_opt_example1} is drawn from standard uniform distribution (each component of the optimization variable is sampled from a uniform distribution over $[0,1]$). For the two dimensional Ackley function minimization defined in the main body of the paper, we use a uniform random initializer over the set $[-5,5]\times [-5,5]$. For  contrastive explanations and adversarial examples on tabular datasets, we draw an initialization from standard uniform distribution (the dimension of the initializer is the same as the dimensionality of the dataset). For contrastive explanations and adversarial examples on MNIST digits data, we found that standard uniform random did not perform well. Instead, we found that sampling points from the training dataset itself worked better. For these settings we used uniform sampling over the training data set. For sum-rate optimization problem, each component of the optimization variable is sampled from a uniform distribution over $[0,1]$

\textbf{MAML initializer.} In \cite{finn2017model}, MAML was proposed for learning models that perform well on few-shot learning tasks. MAML procedure can be understood as a way of generating an effective initialization for a model which quickly adapts to new tasks.  The steps in MAML initialization  applied to the optimization problem in equation \eqref{eqn: stat_opt}  are given in Algorithm \ref{alg1}. The output of the algorithm is the initialization vector we use to solve the problem for new instances. 
For the different methods, we use a batch size $K=32$ for MAML initializer, $\alpha=0.01$, and $\beta=0.01/K$, where $\alpha$ and $\beta$ are the learning rates in the Algorithm \ref{alg1}.
\begin{algorithm}[tb]
  \caption{MAML Initializer}
% \label{alg:VIA}
\begin{algorithmic}
  \STATE {\bfseries Input:} Instances $\{X_i\}_{i=0}^{\infty}$
%   \REPEAT
\STATE $\Pi_{\Theta}$: projection on the set $\Theta$
\STATE  $\theta \leftarrow$ random initialize
\FOR{$i \in \{0,..N-1\}$}
\item  Sample a batch of size $K$ of optimization tasks $\{X_t\}_{t=1}^{K} \in \mathbb{P}_{X}$
\FOR{$t\in \{1,..,K\}$}
\STATE Update $\theta_t = \Pi_{\Theta}(\theta -\alpha\nabla_{\theta}f(\theta, X_t))$
\ENDFOR
\STATE $\theta = \Pi_{\Theta}(\theta -\beta\sum_{t=1}^{K}\nabla_{\theta}f(\theta_t,X_t))$
\ENDFOR
\STATE \textbf{Output:} $\theta$
\end{algorithmic}
\label{alg1}
\end{algorithm} 

\textbf{Val-Init Initializer.} Recall Val-Init Algorithm (described in the main body of the paper) has two phases. In the first phase, we use random initializer (described above) and solve the problem to generate training data for the next phase. We learn a model from a hypothesis class $\mathcal{H}_{\mathsf{val}}$, which we define next. The model for Val-Init is a 3 layer neural network. The shape of the input layer depends on the problem (for generating adversarial examples and contrastive explanations the size of the layer is $2*\mathsf{input}$-$\mathsf{shape}$, where $\mathsf{input}$-$\mathsf{shape}$ is the dimension of each data instance). The two hidden layers have 200 nodes and  ReLU activations. The output layer has one node. 
 The model is trained using mean squared error loss function with Adam optimizer. We set a learning rate of 1e-3 and the other parameters of Adam are set to default value. The total number of training epochs were set to 100.

\textbf{Arg-Init Initializer.} Recall Arg-Init Algorithm (described in the main body of the paper) has two phases. In the first phase, we use random initializer (the choice of random initializer was already described above) and solve the problem to generate training data for the next phase. We learn a model from a hypothesis class $\mathcal{H}_{\mathsf{arg}}$, which we define next. The model for Arg-Init is  a 3 layer neural network.  The shape of the input layer depends on the problem (for generating adversarial examples and contrastive explanations the size of the layer is $2*\mathsf{input}$-$\mathsf{shape}$, where $\mathsf{input}$-$\mathsf{shape}$ is the dimension of each data instance).  The two hidden layers have 200 nodes and  ReLU activations. The output layer has the same number of dimensions as the dimension of each data instance. The model is trained using mean squared error loss function with Adam optimizer. We set a learning rate of 1e-3 and the other parameters of Adam are set to default value. The total number of training epochs were set to 100.
 
\textbf{Remark about $h_{\mathsf{val}}$ and $h_{\mathsf{arg}}$ CIFAR-10.}
When learning  $h_{\mathsf{val}}$ and $h_{\mathsf{arg}}$ for CIFAR-10 dataset, we do not feed raw images to the network of the three layer MLPs described above. Instead, we first extract the representation of these images from the convolutional neural network model (described in Section \ref{sup_mode}) trained to predict the labels in CIFAR-10. We extract these representations  from the layer before the output layer.

\subsection{Solvers and other details}
In all the problems, we use  gradient descent based solvers described in Algorithm \ref{alg1} in the main body of the paper. The gradient descent solver in Algorithm 1 is supplied an input of step update rules to use, and we for our experiments use-- in iteration $k$ take a step length $\frac{p}{q+k}$

\textbf{Convex optimization setting.}  Consider the problem in equation \eqref{eqn: stat_opt_example1}, we set $\beta=1$. We set $u$ to be a linear classifier $u(x) = a^tx$, where $a \in \mathbb{R}^{m}$  and $x \in \mathbb{R}^{m}$ are the parameters that are fixed for an instance of the optimization. We run experiments for $m=50,75,100$. We keep the classifier $a$ fixed (a random normal vector) and vary the input example $x$ across problem instances (each new instance $x$ is also a random normal vector). Our goal is to solve for the optimal perturbation $\theta$. We use projected subgradient descent (PGD) based approach to solve this problem. The projection part is needed to ensure constraints are exactly satisfied. We use subgradients as the loss function involves an $\ell_1$ norm term. We set $p=1$ and $q=25$ in the step update rule.

For training Val-Init and Arg-Init, we need to generate training data (See Algorithm \ref{alg:VIA} in the main body of the paper). We use a random initializer (described in Section \ref{init_sec}) in the first phase of the algorithm. The number of examples for which we solve the problem in Phase 1 of the Algorithm \ref{alg:VIA}, $N$ is set to 1500. In the second phase, we test on 500 instances.  

\textbf{Ackley function minimization.} Ackley functions are used to study non-convex optimization problems \cite{ackley2012connectionist}. The general form of a two dimensional Ackley function is    
\begin{equation}
\begin{split}
    & \mathsf{A(x,y,a,b,c)} = -\mathsf{a exp(-b\frac{\sqrt{(x-c)^2+(y-c)^2}}{2})}  -\mathsf{exp(\frac{cos(2\pi(x-c)) + cos(2\pi(y-c))}{2}) +exp(1) +a}  
    \end{split}
\end{equation}
The global minimum of the function is at $\mathsf{(c,c)}$ and the minimum value is zero.  We are given a sequence of functions with different values of $\mathsf{a,b,c}$, where $\mathsf{a\sim 20 + U[0,10]}$, $\mathsf{b\sim 0.2 + U[0,0.1]}$ and $\mathsf{c\sim U[0,2]}$ and our goal is to learn an initializer that takes $\mathsf{a,b,c}$ as input and generates an initialization. We set $p=0.25$ and $q=1$ in the step update rule. For training Val-Init and Arg-Init, we need to generate data in the first phase. We use a random initializer as described in the previous section. The number of examples for which we solve the problem in Phase 1 of the Algorithm \ref{alg:VIA}, $N$ is set to 1500. In the second phase, we test on 500 instances.

\textbf{Adversarial examples.} Consider the problem in equation \eqref{eqn: stat_opt_example1} in the main body of the paper, we set $\beta=0$. We set $u$ to be a neural network model that we learn for the specific dataset (CIFAR-10, HELOC, MNIST, Waveform). As we explained in the manuscript, we use a penalty based version of the problem in equation \eqref{eqn: stat_opt_example1} in the main body (following the approach in \cite{cheng2018query}). We set the value of the margin, which is a parameter in the penalty term, to be 0.2. We set the penalty value such that the fraction of constraints satisfied in equation \eqref{eqn: stat_opt_example1} by random initialization is high. Hence, the penalty value varies across datasets. For MNIST digits the penalty value is set to be $2.5$. For HELOC and Waveform datasets the penalty value is set to be $10$. For CIFAR-10 dataset the penalty value is set to be $50$. We use a standard gradient descent based solver as the objective is differentiable. For CIFAR-10, MNIST digits we set the step update parameters as follows $p=1$, $q=1$. For tabular datasets (HELOC, Waveform) we set the step update parameters as follows $p=1$, $q=5$. 
Recall we described the split of HELOC, Waveform and MNIST data into train and test, which were used to learn models (Section \ref{sup_mode}).  We use the same training data to generate the adversarial examples (Phase 1 of the Algorithm \ref{alg:VIA} in the main body) and use the same testing data to compare the different initializers.

\textbf{Contrastive explanations.} Consider the problem in equation (2) in the main body, we set $\beta=0.1$ and constrain the perturbations to be valid contrastive explanations (as described in \cite{dhurandhar2018explanations}). We set $u$ to be a neural network model that we learn for the specific dataset (MNIST, Waveform). We consider a penalty based version of the problem in equation \eqref{eqn: stat_opt_example1} in the main body (following the approach in \cite{dhurandhar2018explanations}). We use the same values for the margin and penalties as in the case for adversarial examples. We use subgradient descent based solver (as described in \cite{dhurandhar2018explanations}). For MNIST digits we set the step update parameters as follows $p=1$, $q=1$. For tabular datasets (HELOC, Waveform) we set the step update parameters as follows $p=1$, $q=5$. 
Recall we described the split of HELOC, Waveform and MNIST data into train and test, which were used to learn models (Section \ref{sup_mode}).  We use the same training data to generate the contrastive explanations (Phase 1 of the Algorithm \ref{alg:VIA} in the main body) and use the same testing data to compare the different initializers.

\textbf{Sum-rate optimization.} We use projected gradient descent based solver, where projections are on the set of constraints of the optimization problem, i.e., $[0,1]^{m}$. We set the step update parameters as $p=1$, $q=1$.  For sum-rate optimization, we train using 5000 problem instances for 15 users. Each problem instance is a new independent draw of the channel matrix from a standard uniform random distribution (each element in the matrix is i.i.d.). We learn initializations for Val-Init, Arg-Init, and MAML.  We test on the 500 instances.

Please note that we did not specify the number of iterations that we fix in gradient descent because we show a comparison of the performance with varying number of iterations in the solver.

\subsection{ Results}

  \subsubsection{Convex optimization experiments}

We compare the methods in terms of the average value of the objective function across the test instances (the best method should have the smallest value) for the same number of iterations. In Figure \ref{fig1_n}, we plot the objective function value achieved by the different approaches vs.\ the number of iterations when $m=75$. Arg-Init is able to find a much better solution (35-70\% reduction) when the number of iterations are small. When the iterations increase all the methods perform well (since the problem is convex) and the method eventually finds the minimum. In Figure \ref{fig4c}, we present the distribution of the objective function values for the different methods for $m=75$. Observe that Arg-Init's distribution is skewed towards the left and its performance stochastically dominates the other methods. We also present results for additional 
  cases here for $m=50$ and $m=100$. See the results in Table \ref{table1}, where we show Arg-Init continues to beat other methods. 
  In addition to the comparisons, we provide validation mean square error values for learners $h_{\mathsf{val}}$, $h_{\mathsf{arg}}$ to show that these methods models were able to learn based on the instances supplied to them at the end of Phase 1 of Algorithm 1 in the main body. In Table \ref{table_learn}, we provide initial and final validation mean square error values.
  
  \begin{figure}[H]
  \begin{center}
    \includegraphics[trim = 0cm 0cm 1cm 1.15cm, clip, width=2.5in]{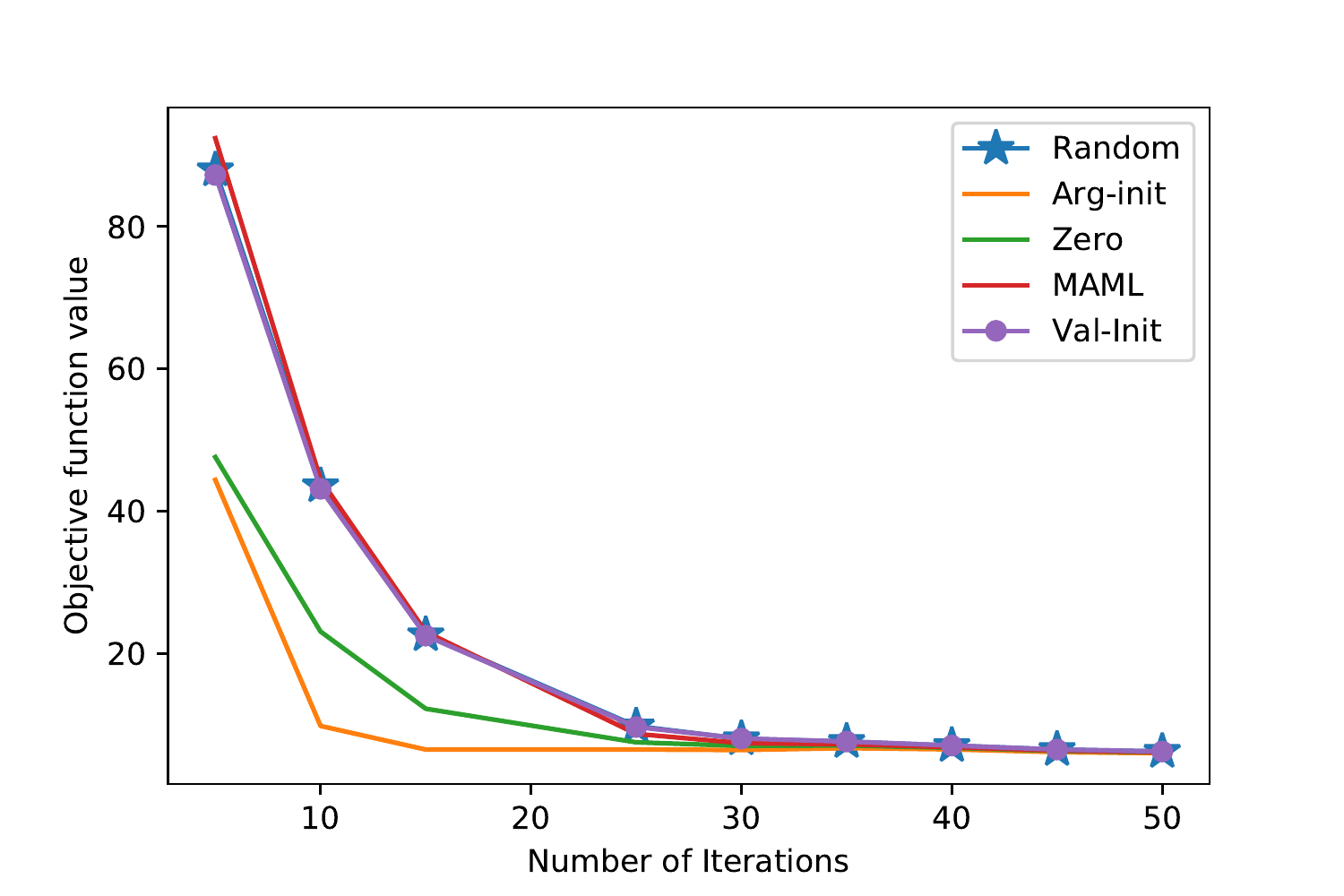}
  \end{center}
  \caption{Objective fun vs.\ number of iterations}
  \label{fig1_n}
\end{figure}

    \begin{table}[H]
\caption{\textbf{Convex optimization setting:} Objective's average value, fraction of instances when constraints not satisfied and the number of iterations per instance in test phase. }
% \label{sample-table}
\begin{center}
\begin{small}
\begin{sc}
\begin{tabular}{lcccr}
\toprule
Method & $m=50$ & $m=75$ & $m=100$ \\
\midrule
Random    & 29.69, 0.0, 10 & 43.56, 0.0, 10  & 57.68, 0.0, 10\\
Zero &  16.07, 0.0, 10 & 23.08, 0.0, 10 &  29.75, 0.0, 10\\  
MAML & 29.84, 0.0, 10 & 44.43, 0.0, 10 & 59.51, 0.0, 10\\
Arg-init &  \textbf{7.70, 0.0, 10} & \textbf{9.82, 0.0, 10} & \textbf{12.75, 0.0, 10}\\
Val-Init & 28.95, 0.0, 10& 43.14,0.0,10 & 57.18.0.0, 10\\ 
\bottomrule
\end{tabular}
\end{sc}
\end{small}
\end{center}
\vskip -0.25in
\label{table1}
\end{table}
  \begin{figure}[H]
\begin{center}
\centerline{\includegraphics[trim = 0cm 0cm 0cm 1cm, clip,width=2.75in]{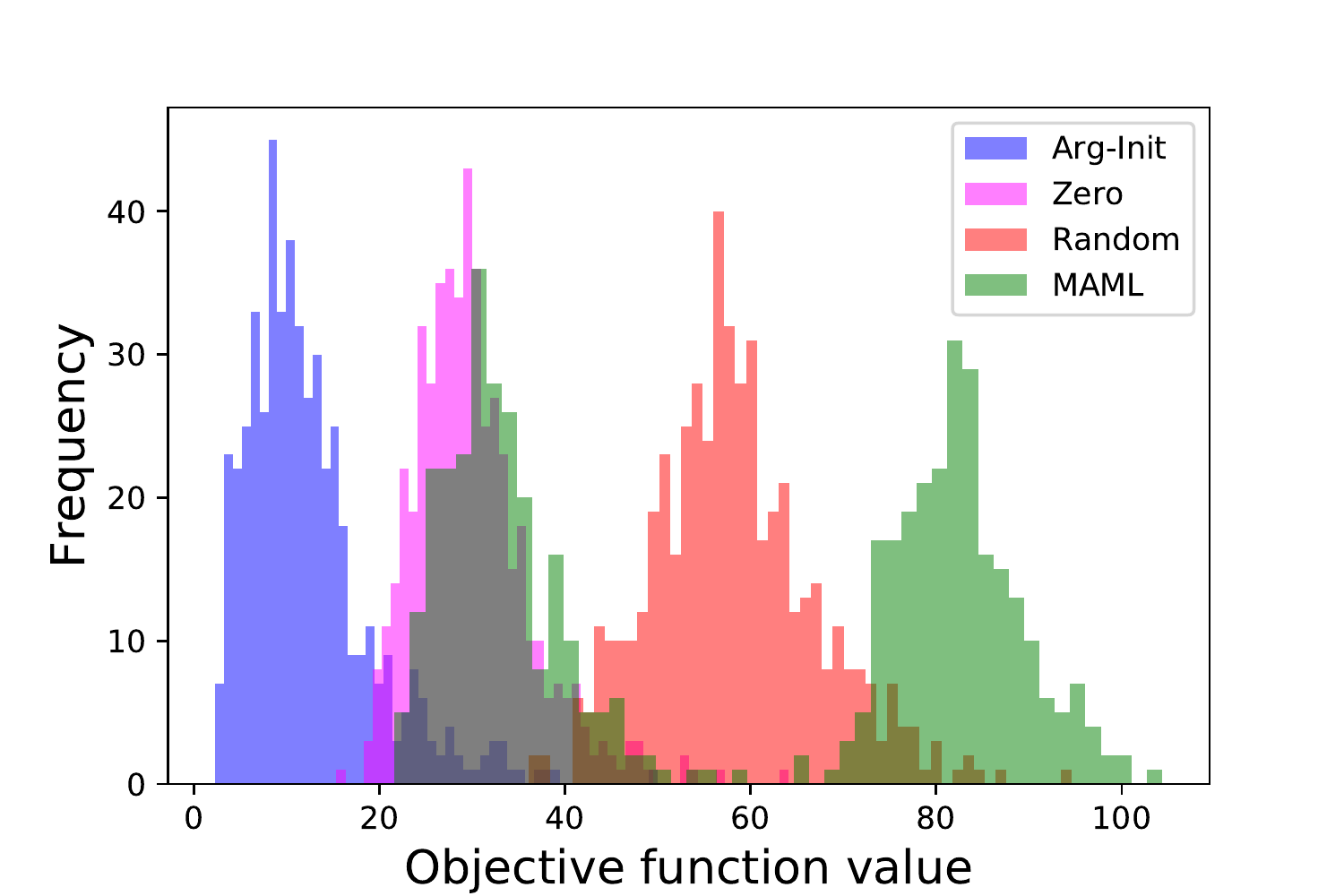}}
\caption{Convex optimization setting: Histogram comparing the objective function value distribution }
\label{fig4c}
\end{center}
\end{figure}

\begin{table}[H]
\caption{\textbf{Convex optimization setting:} Learner errors $h_{\mathsf{val}}$ ($h_{\mathsf{arg}}$)}
\begin{center}
\begin{small}
\begin{sc}
\begin{tabular}{lcccr}
\toprule
Method & $m=50$ & $m=75$ & $m=100$ \\
\midrule
Epoch 1 (first)    &  9.22 (0.026)&  13.34  (0.0020)& 33.06 (0.0015) \\
Epoch 100 (last) & 0.11 (0.010)&  0.24 (0.0004)& 0.02 (0.0005)\\ 
\bottomrule
\end{tabular}
\end{sc}
\end{small}
\end{center}
\label{table_learn}
\end{table}

  \subsubsection{ Ackley function minimization}
  
  In the main body, we presented the experiments for Ackley function minimization: objective function value vs number of iterations. Here we provide the plot of distribution of the objective function values (when number of iterations are fixed to 50). In Figure \ref{fig5}, we observe that the scores of Val-Init lie most to the left. Since the distributions are very skewed, we provide a Zoomed in view in Figure \ref{fig6}. We also provide the validation mean square error  when learning $h_{\mathsf{val}}$ and $h_{\mathsf{arg}}$. The initial (epoch 1) and final (epoch 100) mean square error values for $h_{\mathsf{val}}$ ($h_{\mathsf{arg}}$) are 18.48 (1.46) and 5.52 (0.10), which reflect these models can learn their respective labels well.

  \begin{figure}[H]
% \vskip 0.2in
\begin{center}
\centerline{\includegraphics[trim = 1cm 0cm 1cm 1cm, clip,width=2.75in]{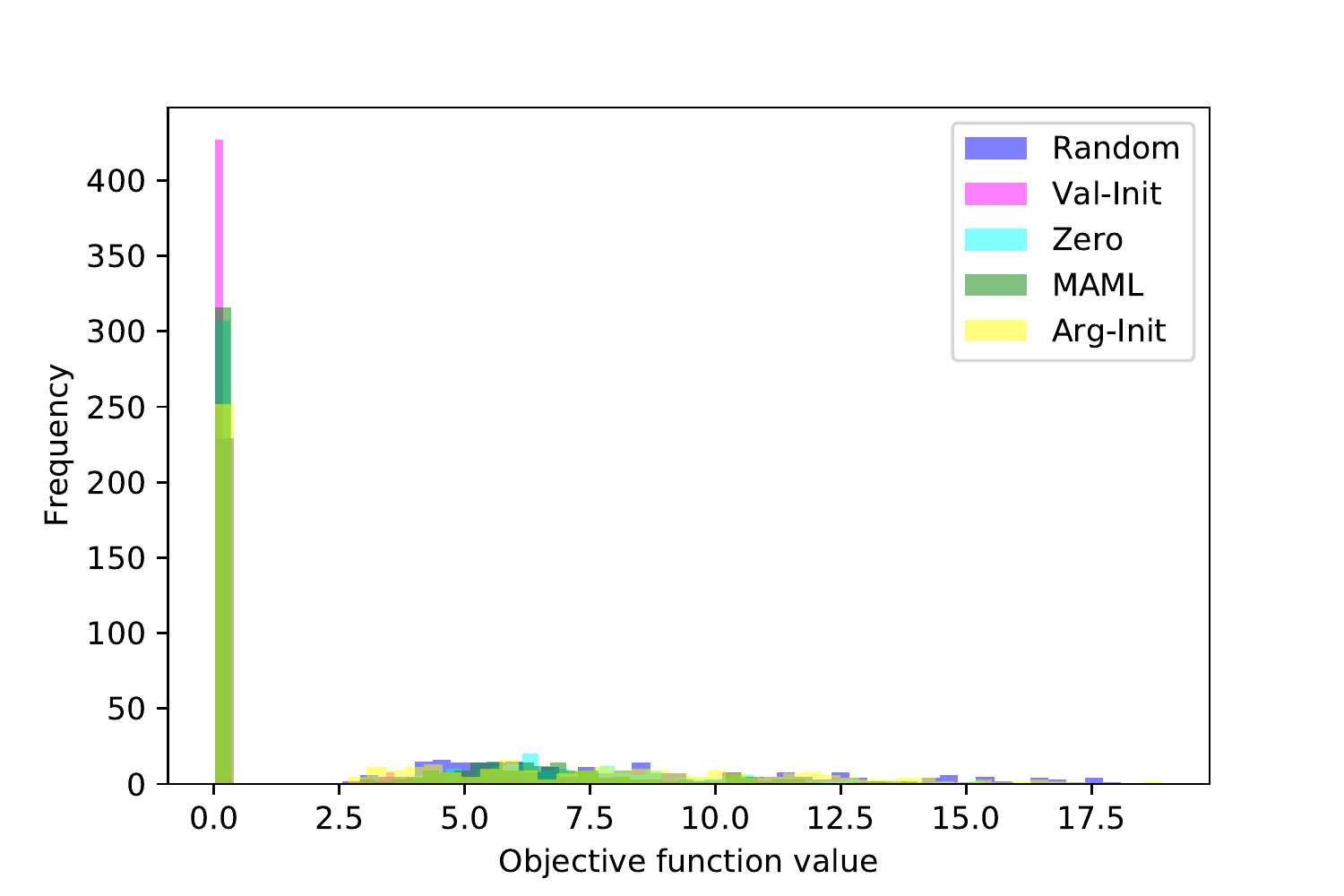}}
\caption{Ackley function minimization: Histogram comparing the objective function value distribution.}
\label{fig5}
\end{center}
% \vskip -0.25in
\end{figure}
\begin{figure}[H]
% \vskip 0.2in
\begin{center}
\centerline{\includegraphics[trim = 1cm 0cm 1cm 1cm, clip,width=2.75in]{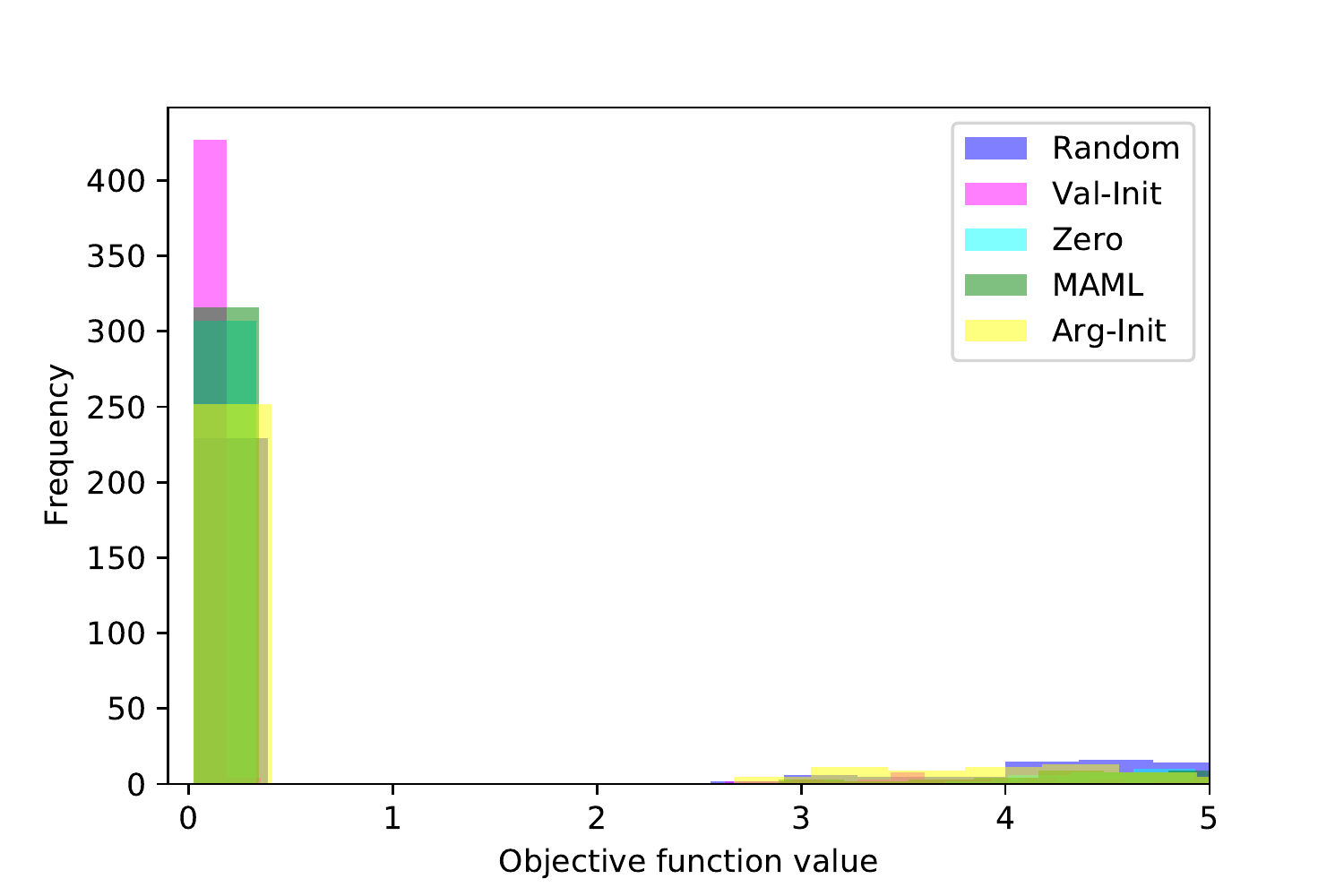}}
\caption{Ackley function minimization: Histogram comparing the objective function value distribution (Zoomed in version). }
\label{fig6}
\end{center}
% \vskip -0.25in
\end{figure}

 \subsubsection{Generating adversarial examples}
  \textbf{CIFAR-10 dataset.} In the main body, we presented the experiments for CIFAR-10 dataset: performance of the method vs. number of iterations. We provide additional details on this comparison below. In Figure \ref{fig22_cifar}, we provide a histogram of distribution of performance values of the methods over different problem instances (for the case when number of iterations is set to $100$). Both Val-Init and Arg-Init have a consistently good performance over all instances (as their distributions towards the left) unlike zero initializer which has an erratic performance distribution -- in some instances it can successfully generate good initializers while in others it is far off.

\begin{figure}[H]
% \vskip 0.2in
\begin{center}
\centerline{\includegraphics[trim = 0cm 0cm 0cm 1cm, clip,width=2.75in]{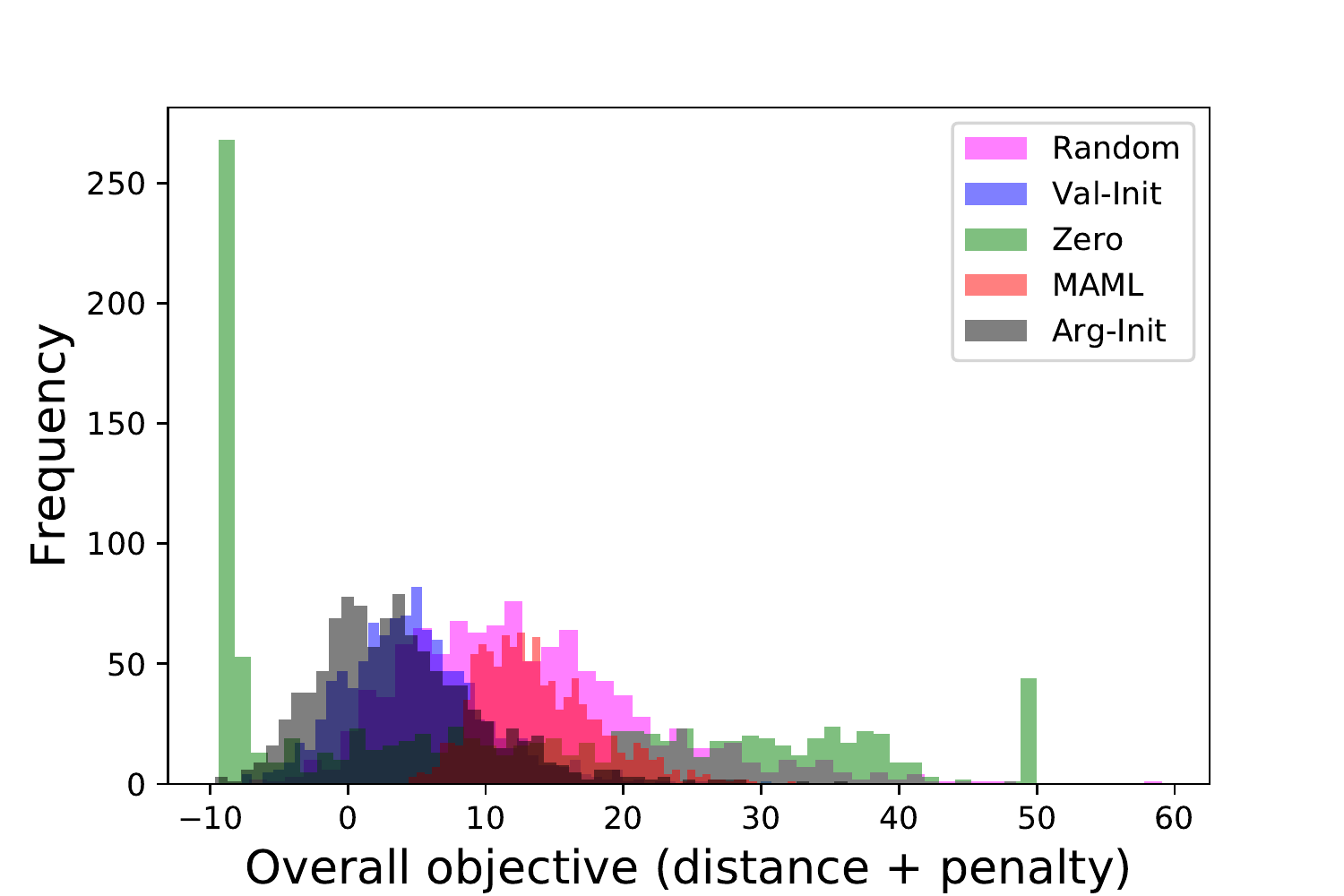}}
\caption{CIFAR-10: Histogram of objective values for Arg-Init, Random, Zero, MAML initialzation }
\label{fig22_cifar}
\end{center}
% \vskip -0.25in
\end{figure}

  \textbf{MNIST Digits.} In the main body in Table \ref{table3}, we showed the average performance of the different initializers on MNIST.  In Figure \ref{fig11_n}, we present the histogram comparing the distribution of performance over the different instances. Since the comparison in Table \ref{table3} in the main body is for a fixed number of iterations, in Figure \ref{fig7_n}, we present the plot the performance of the method vs.\ the number of iterations.  We also provide the validation mean square error  when learning $h_{\mathsf{val}}$ and $h_{\mathsf{arg}}$. For Arg-Init (Val-Init) the MSE at the end of first epoch:  0.0089 (0.79) and at the end of last epoch: 0.0026 (0.33). 

  \begin{figure}[H]
% \vskip 0.2in
\begin{center}
\centerline{\includegraphics[trim = 0cm 0cm 0cm 1cm, clip,width=2.75in]{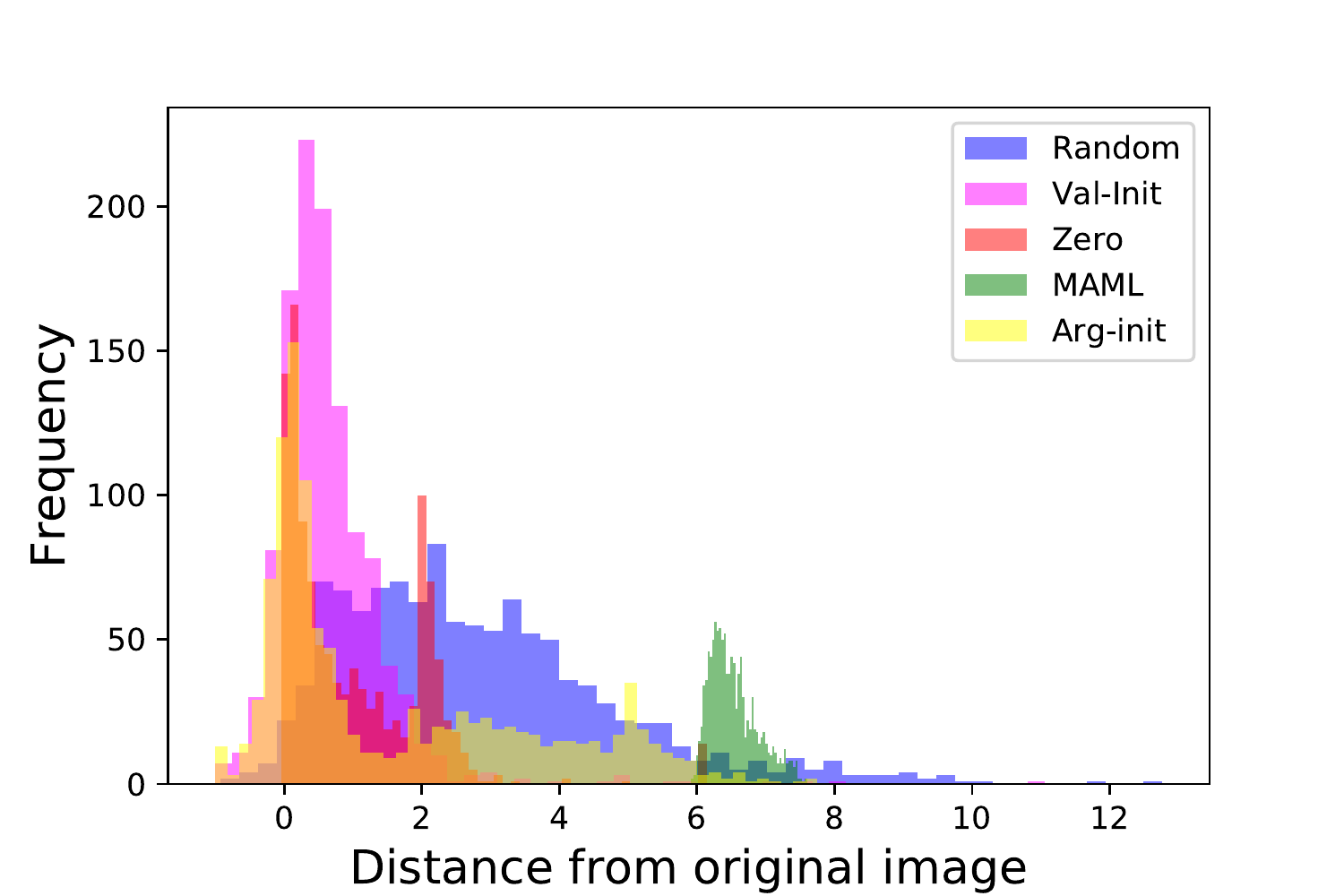}}
\caption{Adversarial example MNIST digits:  Histogram comparing the distribution of the average penalized objective.}
\label{fig11_n}
\end{center}
% \vskip -0.25in
\end{figure}

      \begin{figure}[H]
% \vskip 0.2in
\begin{center}
\centerline{\includegraphics[trim = 0cm 0cm 0cm 0cm, clip,width=2.75in]{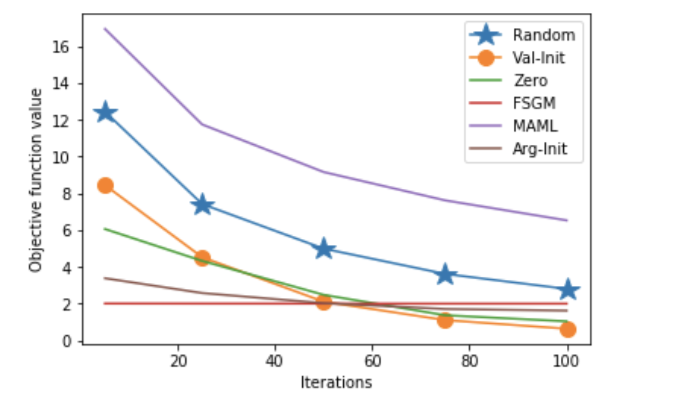}}
\caption{Adversarial example MNIST digits: Compare  performance (average penalized objective function value across test instances) of the method vs. iterations for Val-Init, Arg-Init, Random, Zero, MAML initialzation.}
\label{fig7_n}
\end{center}
% \vskip -0.25in
\end{figure}

  \textbf{Waveform dataset.}  In the main body in Table \ref{table3}, we showed the average performance of the different initializers on Waveform and found Arg-Init, Zero and MAML were very close. In Figure \ref{fig15_n}, we show the distribution of the performance of the methods over the different instances.  Arg-Init, Zero and MAML lie on the left most side of the histogram, while Val-Init and Random are on the right side. In Figure \ref{fig13_n}, we compare the performance of the methods for different number of iterations and we find that MAML performs the best and as number of iterations increase Arg-Init closes in on MAML's performance.  We also provide the validation mean square error  when learning $h_{\mathsf{val}}$ and $h_{\mathsf{arg}}$.
   The initial (epoch 1) and final (epoch 100) mean square error values for $h_{\mathsf{val}}$ ($h_{\mathsf{arg}}$) are 0.0296 (0.0047) and 0.0233 (0.0008), which reflect these models can learn their respective labels well.

  \begin{figure}[H]
% \vskip 0.2in
\begin{center}
\centerline{\includegraphics[trim = 0cm 0cm 0cm 1cm, clip,width=2.75in]{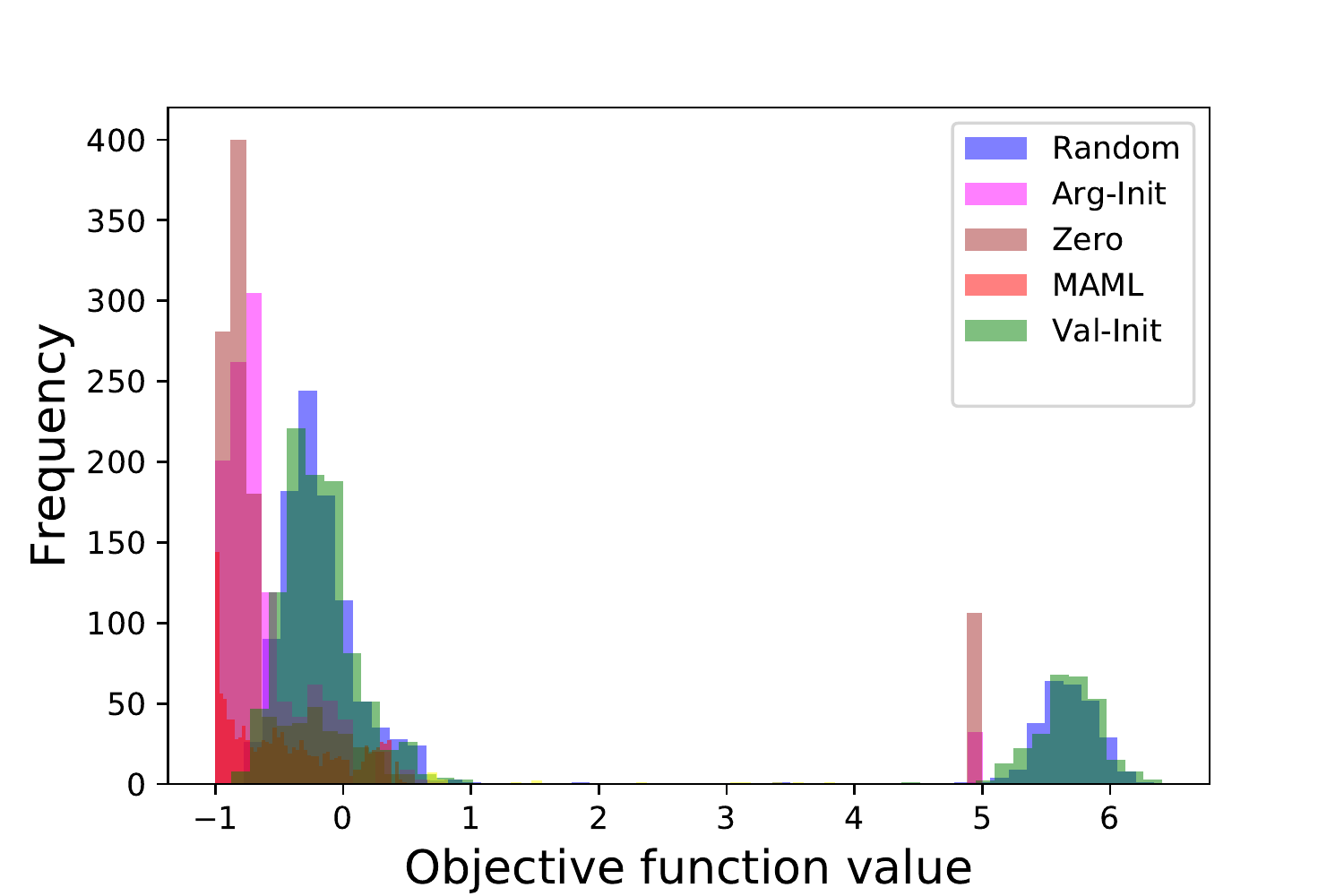}}
\caption{Adversarial example waveform data: Histogram of performance of methods Val-Init, Arg-Init, Random, Zero, MAML initialzation }
\label{fig15_n}
\end{center}
% \vskip -0.25in
\end{figure}

    \begin{figure}[H]
% \vskip 0.2in
\begin{center}
\centerline{\includegraphics[trim = 0cm 0cm 0cm 0cm, clip,width=2.75in]{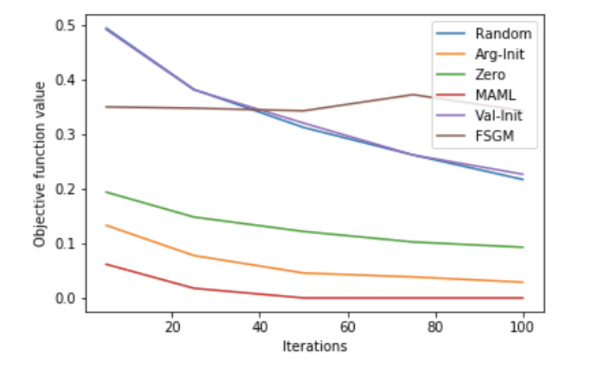}}
\caption{Adversarial example waveform data: Compare the distance from the original point vs iterations for Val-Init, Arg-Init, Random, Zero, MAML initialzation}
\label{fig13_n}
\end{center}
% \vskip -0.25in
\end{figure}

  \textbf{HELOC dataset.} In the main body in Table \ref{table3}, we showed the average performance of the different initializers for a fixed number of iterations. In Figure \ref{fig23_n}, we compare the different methods in terms of performance of methods vs. the number of iterations.   We also provide the validation mean square error  when learning $h_{\mathsf{val}}$ and $h_{\mathsf{arg}}$.
  The initial (epoch 1) and final (epoch 100) mean square error values for $h_{\mathsf{val}}$ ($h_{\mathsf{arg}}$) are 0.0369 (0.0046) and 0.0152 (0.0007), which reflect these models can learn their respective labels well.

\begin{figure}[H]
% \vskip 0.2in
\begin{center}
\centerline{\includegraphics[trim = 0cm 0cm 0cm 0cm, clip,width=2.75in]{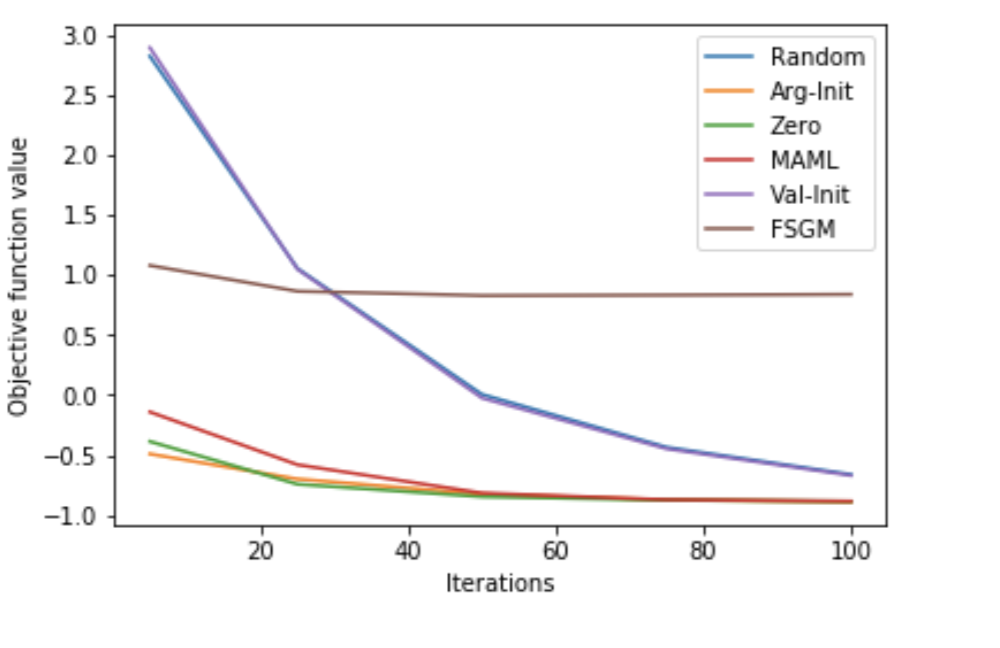}}
\caption{Adversarial example HELOC data: Compare the distance from the original point vs. iterations for Val-Init, Arg-Init, Random, Zero, MAML initialzation}
\label{fig23_n}
\end{center}
%\vskip -0.25in
\end{figure}

  \subsection{Contrastive Explanations}

  \subsubsection{MNIST Digits}
  
  In the main body, we presented the comparison for MNIST digits for a fixed number of iterations. In the main body in Table \ref{table3_n}, we showed the average performance of the different initializers on MNIST.  In Figure \ref{fig18}, we present the histogram comparing the distribution of performance over the different instances. Observe that Arg-Init is the most towards the left. Since the comparison in Table \ref{table3_n} in the main body is for a fixed number of iterations, in Figure \ref{fig17}, we present the plot the performance of the method vs.\ the number of iterations.  We also provide the validation mean square error  when learning $h_{\mathsf{val}}$ and $h_{\mathsf{arg}}$.   For Arg-Init (Val-Init) the MSE at the end of first epoch:  0.0175 (0.097) and at the end of last epoch: 0.0071 (0.057). 

 \begin{figure}[H]
% \vskip 0.2in
\begin{center}
\centerline{\includegraphics[trim = 0cm 0cm 0cm 0cm, clip,width=2.75in]{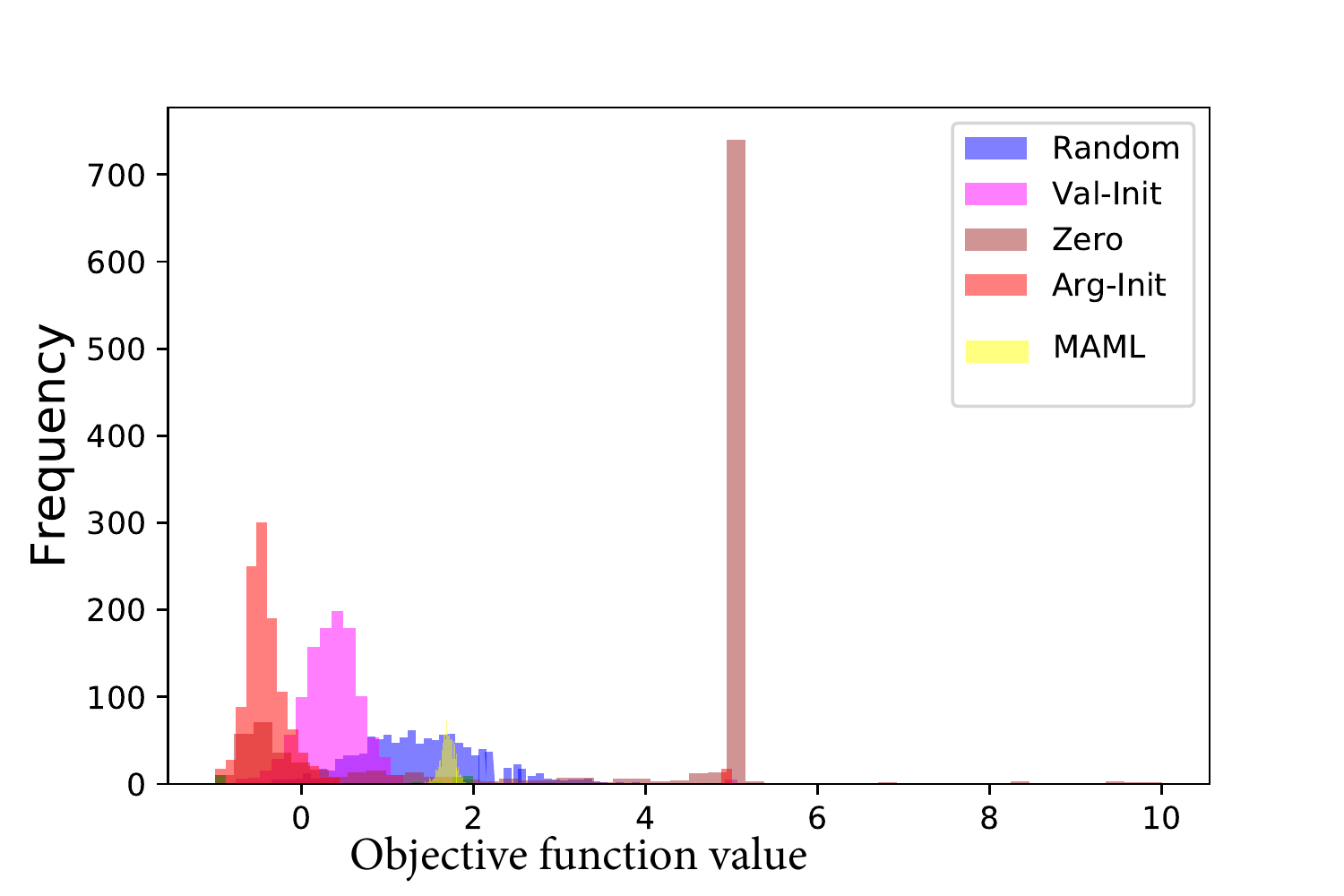}}
\caption{Contrastive explanations MNIST digits:  Compare the distance from the original point vs iterations for Val-Init, Arg-Init, Random, Zero, MAML initialzation.}
\label{fig18}
\end{center}
% \vskip -0.25in
\end{figure}
 \begin{figure}[H]
% \vskip 0.2in
\begin{center}
\centerline{\includegraphics[trim = 0cm 0cm 0cm 0cm, clip,width=2.75in]{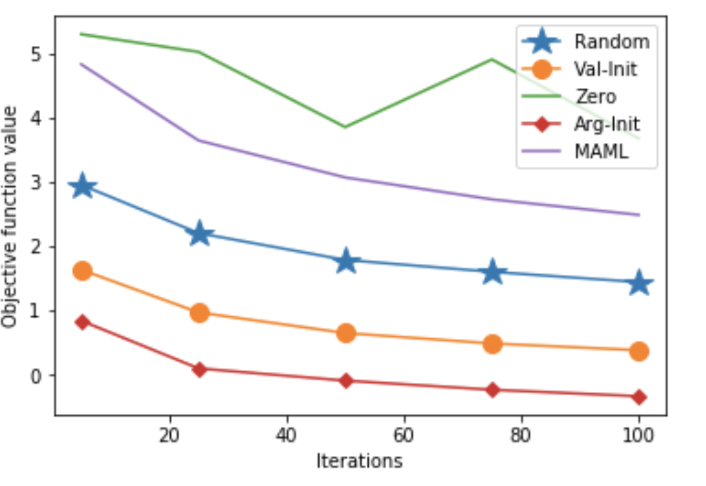}}
\caption{Contrastive explanations MNIST digits:  Compare the performance of method vs iterations for Val-Init, Arg-Init, Random, Zero, MAML initialzation.}
\label{fig17}
\end{center}
% \vskip -0.25in
\end{figure}

  \subsubsection{Waveform dataset}
  In the main body in Table \ref{table3_n}, we showed the average performance of the different initializers.
 In Figure \ref{fig21}, we compare the histogram of distance from the original point (for iterations fixed to $100$). Arg-Init is the most skewed towards the left compared to the other distributions. 
   In Figure \ref{fig20}, we compare the objective function value vs. the iterations.
 We also provide the validation mean square error  when learning $h_{\mathsf{val}}$ and $h_{\mathsf{arg}}$. 
  For Arg-Init (Val-Init) the MSE at the end of first epoch:  $0.027$ ($6.99$) and at the end of last epoch: $0.0076$ ($2.69$). 

  \begin{figure}[H]
% \vskip 0.2in
\begin{center}
\centerline{\includegraphics[trim = 0cm 0cm 0cm 0cm, clip,width=2.75in]{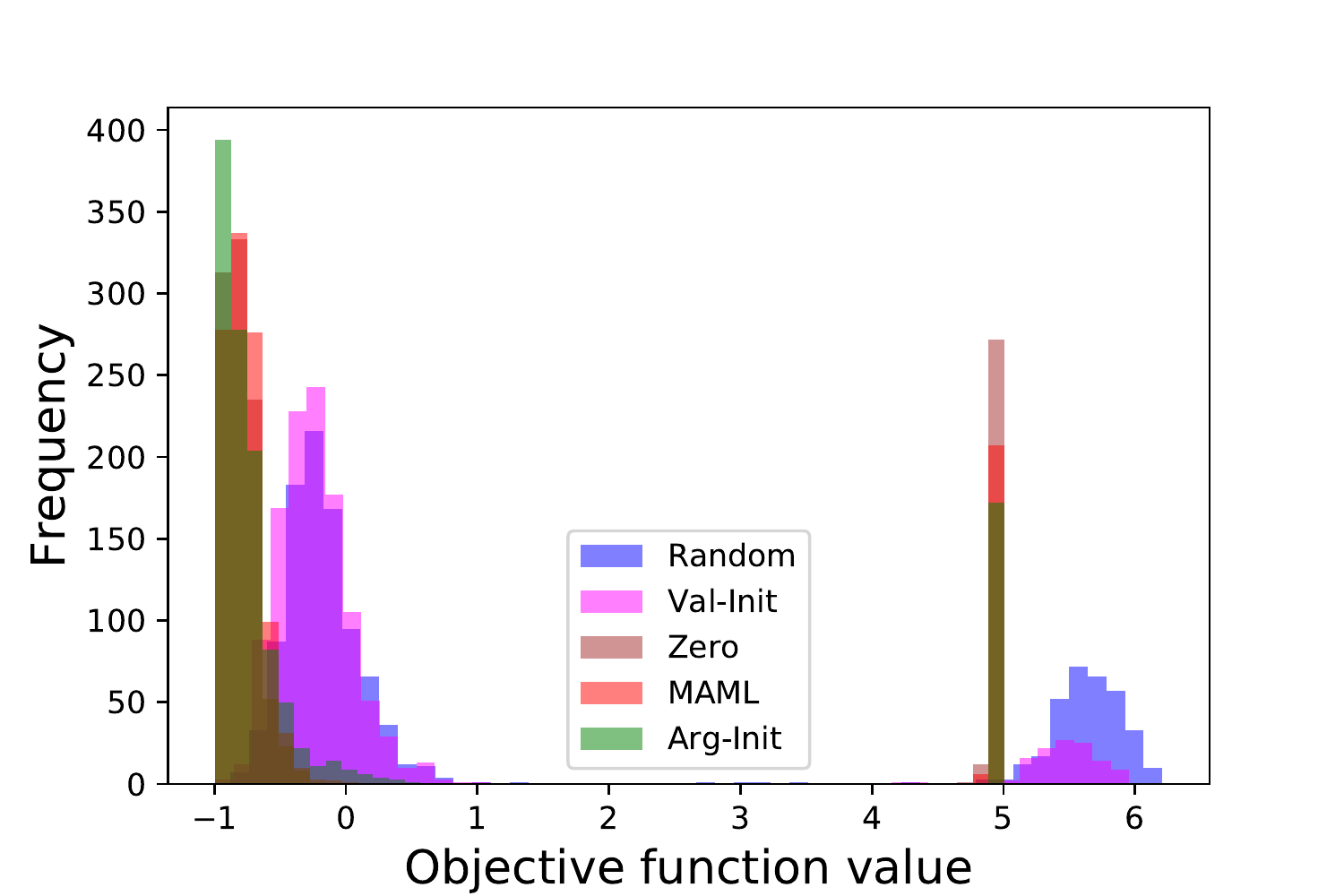}}
\caption{Contrastive explanations Waveform dataset:  Compare the performance of method vs iterations for Val-Init, Arg-Init, Random, Zero, MAML initialzation.}
\label{fig21}
\end{center}
% \vskip -0.25in
\end{figure}
 \begin{figure}[H]
% \vskip 0.2in
\begin{center}
\centerline{\includegraphics[trim = 0cm 0cm 0cm 0cm, clip,width=2.75in]{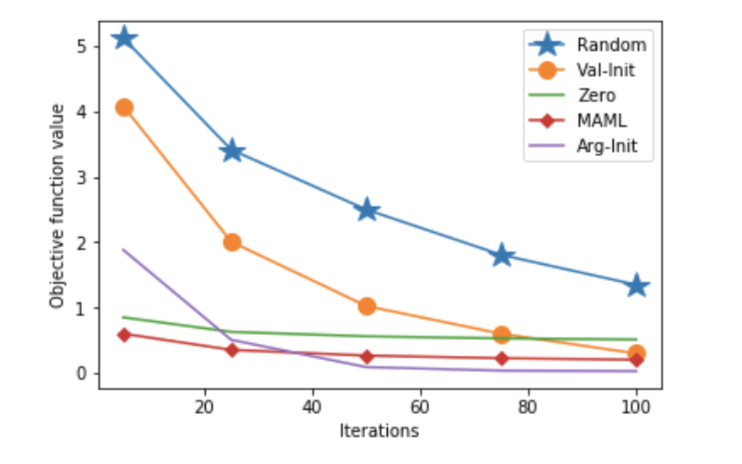}}
\caption{Contrastive explanations Waveform dataset:  Compare the performance of method vs iterations for Val-Init, Arg-Init, Random, Zero, MAML initialzation.}
\label{fig20}
\end{center}
% \vskip -0.25in
\end{figure}

\subsubsection{Sum-rate optimization}
% \vspace{-0.5em}
In the main body we showed how for sum-rate optimization problem Arg-Init was the best performing method in terms of the average performance across different number of iterations.  We fix the iterations to 100 and  show the distribution of the performance of the methods across different problem instances in Figure \ref{fig2}. As we can see from the Figure \ref{fig2} the distribution of Arg-Init is the most towards the left among all the distributions.  We also provide the validation mean square error  when learning $h_{\mathsf{val}}$ and $h_{\mathsf{arg}}$. 
  For Arg-Init (Val-Init) the MSE at the end of first epoch:  $0.100$ ($0.0200$) and at the end of last epoch: $0.0135$ ($0.0113$).

\begin{figure}[ht]
% \vskip 0.2in
\begin{center}
\centerline{\includegraphics[trim=1cm 0cm 1.5cm 0.5cm,clip=true,width=2in]{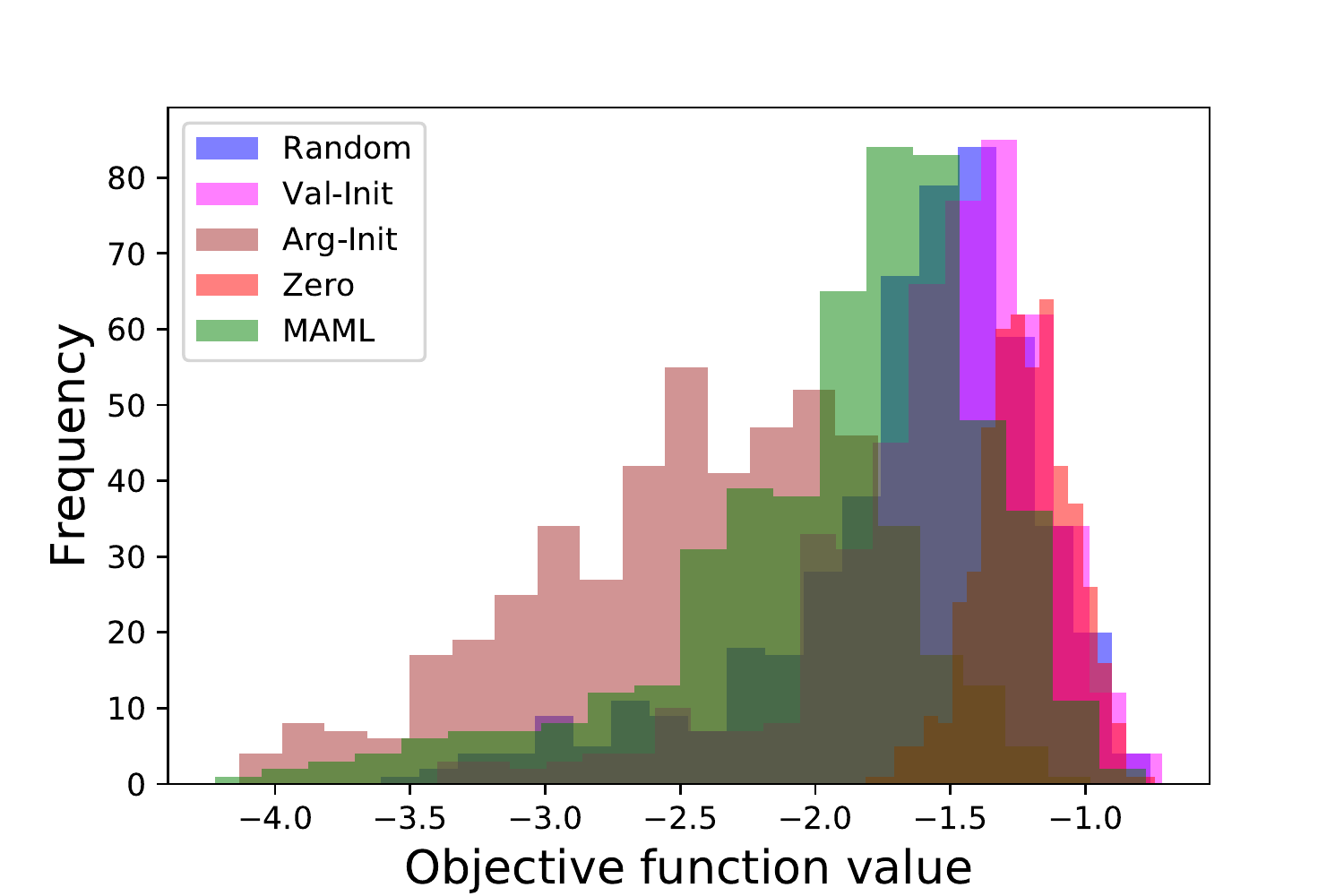}}
\caption{Sum rate optimization: Histogram comparing objective function values }
\label{fig2}
\end{center}
\vskip -0.1in
\end{figure}

\section{Acknowledgement}
We would like to acknowledge Karthikeyan Shanmugam for the valuable comments and feedback on the work.

\bibliographystyle{apalike}
\bibliography{aistats_learn_init}

\end{document}